\documentclass{article} 
\usepackage{iclr2026_conference,times}

\usepackage{amsmath,amsfonts,bm}

\def\eqref#1{equation~\ref{#1}}

\def\1{\bm{1}}

\DeclareMathAlphabet{\mathsfit}{\encodingdefault}{\sfdefault}{m}{sl}
\SetMathAlphabet{\mathsfit}{bold}{\encodingdefault}{\sfdefault}{bx}{n}

\usepackage{hyperref}
\usepackage{makecell}
\usepackage{url}
\usepackage{amsmath}
\usepackage{amssymb}
\usepackage{amsthm}
\usepackage{amsfonts}
\usepackage{bbm}
\usepackage{amssymb}
\usepackage{wrapfig}
\usepackage{caption}
\usepackage{graphicx}
\usepackage{subfigure}
\usepackage{booktabs}
\usepackage{adjustbox}
\usepackage{subcaption}
\usepackage{algorithmic}
\usepackage[ruled]{algorithm2e}
\usepackage{tcolorbox}

\newtheorem{lemma}{Lemma}

\newtheorem{theorem}{Theorem}

\newtheorem{definition}{Definition}

\title{Enhancing Generative Auto-bidding with \\ Offline Reward Evaluation and Policy Search}

\author{%
  Zhiyu Mou$^{1}$\thanks{Equal contribution; \textsuperscript{$\dag$}Work was done during an internship at Alibaba Group;
  \textsuperscript{$\S$}Corresponding authors.} \quad 
  Yiqin Lv$^{1,2*\dag}$ \quad 
  Miao Xu$^{1}$ \quad 
  Qi Wang$^{2\S}$ \quad 
  Yixiu Mao$^{2}$ \quad 
  Jinghao Chen$^{1,2\dag}$ \and 
  \textbf{Qichen Ye}$^{1}$ \quad
   \textbf{Chao Li}$^{1}$ \quad 
   \textbf{Rongquan Bai}$^{1\S}$ \quad
   \textbf{Chuan Yu}$^{1}$ \quad 
   \textbf{Jian Xu}$^{1}$ \quad 
   \textbf{Bo Zheng}$^{1}$ 
   \\
    $^{1}$Taobao \& Tmall Group of Alibaba, Beijing, China\\
  $^{2}$Department of Automation, Tsinghua University, Beijing, China\\
  \texttt{mouzhiyu.mzy@taobao.com, hhq123go@gmail.com, rongquan.br@taobao.com}}

\iclrfinalcopy

\begin{document}

\maketitle
\vspace{-2mm}
\begin{abstract}
Auto-bidding is a critical tool for advertisers to improve advertising performance. 
Recent progress has demonstrated that AI-Generated Bidding (AIGB), which learns a conditional generative planner from offline data, achieves superior performance compared to typical offline reinforcement learning (RL)-based auto-bidding methods. 
However, existing AIGB methods 
still face a performance bottleneck due to their inherent inability to explore beyond the static dataset with feedback.
To address this, we propose \textbf{AIGB-Pearl} (\emph{\textbf{P}lanning with \textbf{E}valu\textbf{A}tor via \textbf{RL}}), 
a novel method that integrates generative planning and policy optimization. 
The core of AIGB-Pearl lies in constructing a trajectory evaluator to assess the quality of generated scores and designing a provably sound KL-Lipschitz-constrained score-maximization scheme to ensure safe and efficient exploration beyond the offline dataset.
A practical algorithm that incorporates the synchronous coupling technique is further developed to ensure the model regularity required by the proposed scheme.
Extensive experiments on both simulated and real-world advertising systems demonstrate the state-of-the-art performance of our approach.

\end{abstract}

\section{Introduction}
\label{sec:intro}

The increasing demand for commercial digitalization has facilitated the development of the auto-bidding technique in online advertising.
Distinguished from traditional manual bidding products, auto-bidding provides advertisers with an efficient and flexible scheme to automatically optimize bids
in dynamic and competitive environments \citep{balseiro2021robust,deng2021towards,balseiro2021landscape}.
Technically, auto-bidding constitutes an offline sequential decision-making problem that aims to maximize advertising performance over a bidding episode, constrained to a static offline dataset due to operational safety concerns \citep{sorl}. 

As a standard approach to offline decision-making problems, offline reinforcement learning (RL) \citep{kumar2020conservative,mao2024doubly} is widely adopted to solve the auto-bidding problem.
By employing conservative policy search schemes, offline RL mitigates the infamous \emph{out-of-distribution} (OOD) problem \citep{fujimoto2019off}, enabling reliable generalization beyond the offline dataset.
However, their reliance on bootstrapped value estimates renders offline RL methods prone to training instability \citep{peng2024deadly}, potentially compromising policy performance.

Recent advances in generative models shed new light on offline decision-making problems \citep{zhu2023diffusion,kang2023efficient}.
Specifically,
AI-generated bidding (AIGB) models auto-bidding as a trajectory-generation task and employs a generative model to approximate the conditional trajectory distribution of the offline dataset \citep{aigb}. AIGB avoids bootstrapping and exhibits more stable training and superior performance.
However, the modeling approach in AIGB does not explicitly align with the performance-maximization of the auto-bidding problem. 
Inherently, AIGB relies on imitating trajectories from the offline dataset \citep{ajay2023is} and lacks the ability to improve its generation quality beyond the offline dataset based on the performance feedback. 
Consequently, its conditional generation in the extrapolation regime can become unreliable, potentially leading to performance degradation or even to risky trajectory generations. 


Hence, there arises a question: \textit{built on AIGB, the latest state-of-the-art auto-bidding method, can we devise a plausible scheme to involve policy optimization in its generative model?}
To this end, a natural idea is to integrate offline RL methods into AIGB. 
However, it is nontrivial to implement in the auto-bidding problem since (i) there is a lack of reward signals in AIGB to guide the generative model. 
Specifically, the generation quality of the generative model remains unknown during training,  making it infeasible to explore new trajectories beyond the offline dataset; (ii) no dedicated offline RL algorithm exists for AIGB.
In particular, theoretical analysis that guarantees safe generalization and mitigates OOD issues for generative models in auto-bidding remains largely unexplored.

To address these critical challenges, we propose \textbf{AIGB-Pearl} (Planning with EvaluAtor via RL), an RL-enhanced version of AIGB that learns a \emph{trajectory evaluator} to score generation quality and drive exploration of the generative model through continuous interaction.
The evaluator is trained through supervised learning on the offline dataset. 
Crucially, to mitigate the OOD problem and ensure reliable use of the evaluator, we examine the theoretical upper bound on the evaluator's bias. Then, guided by this analysis, we formulate a \emph{KL-Lipschitz-constrained} score-maximization objective with a provable suboptimality bound, enabling safe and effective exploration beyond the offline data.
Moreover, to perform constrained score maximization, we design a practical algorithm that incorporates the \emph{synchronous coupling} technique to satisfy the generative model's Lipschitz condition. 
In addition, we note that AIGB-Pearl does not require bootstrapping and exhibits greater training stability than offline RL methods.

To summarize, our contributions in this paper are fourfold: (i) we propose a novel generative auto-bidding method, AIGB-Pearl, that enables continuous improvement in generation quality through exploration beyond the offline dataset; (ii) we propose a provable KL-Lipschitz constrained score maximization objective with a sub-optimality bound, ensuring a safe and effective generalization beyond the offline dataset; (iii) we devise a practical algorithm with synchronous coupling that effectively ensures the Lipschitz requirement for the generative model; (iv) extensive simulated and real-world experiments demonstrate that AIGB-Pearl achieves SOTA performance and verify the effectiveness of the developed techniques in enhancing safe and effective generalization.

\vspace{-2mm}
\section{Preliminaries}
\label{sec:preliminary}
\vspace{-1mm}
\subsection{Problem Statement}
\label{sec:prob_state}
\vspace{-1mm}
This work studies the auto-bidding problem for a single advertiser subject to a budget constraint.
The auction mechanism follows a sealed-bid, second-price rule. The objective is to devise a bidding policy that maximizes the cumulative value of the impressions won over a finite bidding episode (e.g., a day) within a budget $B>0$.
As established in \citep{uscb}, the optimal bid for each impression $i$ is proportional to its intrinsic value $v_i>0$, scaled by a non-negative factor $a\ge 0$ that remains consistent across all impressions.
Under this strategy, the advertiser wins an impression $i$ if $av_i\ge p_i$ and pays $p_i$ upon winning, where $p_i>0$ is the market price.
The Return on Investment (ROI) of impression $i$ is defined as $v_i/p_i$, and we denote its upper bound as $R_m\triangleq\max_i v_i/p_i$.

However, the scaling factor is unknown a prior, and impression volatility drives its continual change throughout the bidding process.
Hence, a standard practice involves recalibrating the scaling factor $a$ at fixed intervals of $T\in\mathbb{N}_+$ time steps \citep{aigb, uscb, sorl}. 
This casts the auto-bidding to a sequential decision-making problem.

Specifically, the auto-bidding problem can be modeled as a Markov Decision Process (MDP) $<\mathcal{S}, \mathcal{A}, \mathcal{R}, \mathcal{P}>$.
The state $s_t\triangleq[t, \bar{c}_{t-1},x]\in\mathcal{S}$ is composed of the current time step $t\in [T]$, the {cost ratio $\bar{c}_{t-1}=c_{t-1}/B>0$} where $c_{t-1}$ is the advertiser's cost for impressions won between time step $t-1$ and $t$, and a static advertiser-specific feature $x$ that {includes the budget and many other individual information}.
The action $a_t\in\mathcal{A}$ denotes the calibrated scaling factor at time step $t$.
The reward $r_t\ge 0$ describes the value of the impressions won between time steps $t$ and $t+1$, and $\mathcal{P}$ denotes the state transition rule. 
The auto-bidding problem can be formulated as:
\begin{align}
\label{equ:prob_formulation}
    \max_{a_1,a_2,\cdots,a_T} \;\mathbb{E}_{ s_{t+1}\sim\mathcal{P}(\cdot|s_t,a_t)}\bigg[\sum_{t=1}^{T} r_t\bigg], \;\;\mathrm{s.t.} \sum_{t=1}^{T} c_t\le B.
\end{align}

\textbf{Offline Setting.}
Due to safety concerns---common in real-world advertising systems---we are restricted to learning the optimal bidding policy from a static \emph{offline dataset} $\mathcal{D}$ comprising historical states and actions along with associated rewards.
This makes the considered auto-bidding problem an offline sequential decision-making task.

\vspace{-1mm}
\subsection{Offline RL Methods}
\vspace{-1mm}
RL constitutes a standard approach for auto-bidding problems, seeking an optimal bidding policy $\pi:\mathcal{S}\rightarrow\mathcal{A}$ that maximizes cumulative reward. 
Specifically, this is typically achieved by learning a Q-value function, $Q(s_t,a_t)\triangleq\mathbb{E}_{\pi}[\sum_{t'=t}^Tr_{t'}]$, through temporal difference (TD) error minimization:
\begin{align}
\label{equ:q_value}
   \min_{Q} \; \mathbb{E}_{(s_t,a_t,r_t,s_{t+1})\sim\mathcal{D}}[Q(s_t,a_t)-r_t -\max_{a_{t+1}}\hat{Q}(s_{t+1},a_{t+1})]^2,
\end{align}
where $\hat{Q}$ is a target Q-value function with parameters updated via Polyak averaging \citep{mnih2015human}.
Upon convergence, the optimal bidding policy is derived as $\pi(s_t)=\arg\max_{a_t}Q(s_t,a_t)$.

Due to the offline setting of the considered auto-bidding problem, directly employing Eq. \ref{equ:q_value} results in the infamous \emph{out-of-distribution} (OOD) problem \citep{fujimoto2019off}, making the policy erroneously deviate from the offline dataset $\mathcal{D}$ \citep{mao2024offline}. 
As a standard solution, offline RL \citep{yu2020mopo,kumar2020conservative,mao2023supported,kidambi2020morel} constrains the policy's behavior near $\mathcal{D}$ during TD learning,  enabling reliable generalization beyond the offline dataset.

However, offline RL methods are notoriously unstable due to training instability caused by TD-learning \citep{peng2024deadly}, in which the bootstrapped value of the Q-function serves as its training label, resulting in an erroneous ground truth.
Training stability is critical in auto-bidding due to the absence of an accurate offline policy evaluation method and the high cost of online policy evaluation in real-world advertising systems \citep{sorl}.
\vspace{-1mm}
\subsection{Generative Auto-bidding Methods}
\vspace{-1mm}
\begin{definition}[Trajectory and Trajectory Quality]
    The \emph{trajectory} is formalized as the state sequence throughout the bidding episode, i.e., $\tau\triangleq [s_1,s_2,\cdots,s_T]$. 
    The trajectory quality is defined as the normalized cumulative reward of the trajectory, i.e., $y(\tau)\triangleq\sum_{t=1}^{T}\bar{r}_t$ \footnote{Note that, as in real-world advertising systems, the bidding process will automatically suspend once the advertiser's budget runs out, and thereby, any action sequence will not violate the budget constraint.
Hence, the trajectory quality can be directly defined as the cumulative reward of the trajectory.}, where $\bar{r}_t=r_t/B$.
\end{definition}

Unlike RL methods, the AI-generated auto-bidding (AIGB) \citep{aigb} 
treats the auto-bidding problem as a sequence generation task.
Specifically, a conditional generative model is employed to fit the conditional trajectory distribution $p_\theta(\tau|y(\tau))$ within the offline dataset $\mathcal{D}$, i.e., 
\begin{align}
\label{equ:mle}
   \max_\theta \; \mathbb{E}_{(\tau,y(\tau))\sim\mathcal{D}}[\log p_\theta(\tau|y(\tau))],
\end{align}
where $\theta$ denotes the parameter.
Let $y_m>0$ be the maximum trajectory quality in $\mathcal{D}$, we have
$\forall y\in\mathcal{D}, y\in[0, y_m]$.
During inference, AIGB follows a \emph{planning-and-control} architecture. Specifically, at each time step, a trajectory is sampled from the trained generative model that acts as the \emph{planner}, with a manually set condition $y^*\triangleq(1+\epsilon)y_m$, where $\epsilon>0$ is a hyper-parameter.
Then, an extra off-the-shelf inverse dynamic model \citep{agrawal2016learning}, acting as the \emph{controller}, is employed to compute the action. 
See Appendix \ref{app:aigb} for detailed descriptions.
AIGB avoids TD learning and generally outperforms offline RL methods \citep{aigb}.

However, AIGB relies on imitating trajectories from the offline dataset \citep{ajay2023is} and lacks the ability to improve its generation quality beyond the offline dataset based on the performance feedback. 
Consequently, AIGB's conditional generation in the extrapolation regime ($y^*>y_m$) can be unreliable without explicit reward guidance, rendering exploration undirected and potentially leading to performance degradation or even risky trajectory generation.
Furthermore, no theoretical guarantee exists for the quality of the generated trajectory in this extrapolation regime.



\vspace{-1mm}
\section{Method}
\label{sec:method}
\vspace{-1.3mm}
Enabling AIGB to explore higher-quality trajectories beyond the offline dataset with explicit reward guidance can enhance its performance and generalization.
To this end, we propose \textbf{AIGB-Pearl} (\textbf{P}lanning with \textbf{E}valu\textbf{A}tor via \textbf{RL}) that constructs a \emph{trajectory evaluator} (hereinafter referred to as the \emph{evaluator} for simplicity) to integrate RL methods into AIGB's planner training.
Specifically, the evaluator learns a \emph{score} $\hat{y}_\phi(\tau)$ to estimate the trajectory quality $y(\tau)$ via supervised learning based on the offline dataset $\mathcal{D}$, i.e., $\min_\phi\mathbb{E}_{\tau\sim\mathcal{D}}[(\hat{y}_\phi(\tau)-y(\tau))^2]$, where $\phi$ denotes the evaluator parameter.
Then, with $\phi$ fixed, the planner tries to maximize the score of its generation through iterative interactions with the evaluator, as shown in Fig. \ref{fig:aigb_pearl}.
Formally, this can be formulated as:
\begin{align}
\label{equ:max_R}
\max_\theta \; L(\theta)\triangleq\mathbb{E}_{\tau\sim p_\theta(\tau|y^*)}[\hat{y}_\phi(\tau)],
\end{align}
where the condition is fixed to $y^*$ in both training and inference stages to ensure consistency.


As shown in Eq. \ref{equ:max_R}, the effectiveness of AIGB-Pearl hinges critically on the evaluator's reliability. 
However, given the offline nature of the considered auto-bidding problem, the evaluator training is confined to the fixed dataset $\mathcal{D}$. 
Although we incorporate several techniques into the evaluator’s supervised training to enhance its prediction accuracy (as detailed in Section~\ref{sec:evaluator_design} and Appendix~\ref{app:aigb_pearl_evaluator}), directly applying Eq. \ref{equ:max_R} to the planner can still trigger the infamous OOD problem due to the evaluator's generalization limits outside the data support, potentially degrading the planner's true performance. 
This issue is particularly acute in auto-bidding, a risk-sensitive domain in which suboptimal or anomalous trajectory generation can result in substantial monetary losses or campaign failures.
Notably, there remains a lack of theoretically principled approaches to this challenge.

To address this challenge, we first analyze the theoretical bounds on the evaluator's bias in Section~\ref{sec:score_max}. Then, guided by this analysis, we propose a KL-Lipschitz-constrained score-maximization objective for the planner to ensure reliable use of the evaluator. 
Notably, this objective is theoretically justified by a sub-optimality bound established in Section \ref{sec:sub_optimality_bound}.
Finally, a practical algorithm is presented in Section \ref{sec:practical_algorithm}, which employs a synchronous coupling method to satisfy the planner's Lipschitz constraint in Section \ref{sec:planner_loss}.

\subsection{KL-Lipschitz-constrained score maximization}
\label{sec:score_max}
This section focuses on the reliable exploitation of the evaluator-guided score maximization. 

\textbf{Our basic idea} is to optimize $\theta$ within a
domain where the gap between the planner's score $L(\theta)$ and its true performance $J(\theta)\triangleq\mathbb{E}_{\tau\sim p_\theta(\tau | y^*)}[y(\tau)]$ is bounded by a small certifiable upper bound. This ensures the score maximization occurs only in regions where the evaluator is reliable. 
\begin{align}
\label{equ:gap}
   |J(\theta)-L(\theta)| =|\mathbb{E}_{\tau\sim p_\theta(\tau | y^*)}[y(\tau)]-\mathbb{E}_{\tau\sim p_\theta(\tau | y^*)}[\hat{y}_\phi(\tau)]|.
\end{align}
In the following, we investigate this gap.
Specifically, we find that the trajectory quality $y(\tau)$ is Lipschitz continuous as stated in Theorem \ref{thm:lipschitz_continuous}. The proof is given in Appendix \ref{app:proof_thm1}.
\begin{theorem}[Lipschitz Continuous of $y(\tau)$.]
\label{thm:lipschitz_continuous}
The trajectory quality $y(\tau)$ is $\sqrt{T}R_m$-Lipschitz continuous with respect to the Frobenius norm.
\end{theorem}
{Motivated by Theorem \ref{thm:lipschitz_continuous}, we enforce a $\sqrt{T}R_m$-Lipschitz regularity on the evaluator's training to inherit the Lipschitz continuity of the true trajectory quality $y(\tau)$} (as described in Section \ref{sec:evaluator_design}). 
As the trained evaluator's Lipschitz constant may not equal $\sqrt{T}R_m$ exactly, we denote it as $k\sqrt{T}R_m$, where $k\ge 0$ quantifies the degree of violation.
Note that a tighter adherence of the evaluator to the Lipschitz constraint yields a value of $k$ closer to 1.

Equipped with Theorem \ref{thm:lipschitz_continuous} and the Lipschitz property of the evaluator, we derive the following upper bound on the performance gap between $J(\theta)$ and $L(\theta)$, and the proof is given in Appendix \ref{app:proof_thm2}.
\begin{theorem}[Evaluator Bias in Planning Performance Bound]
\label{thm:performance_gap_bound}
Let the upper bound of the evaluator's bias on its training dataset $\mathcal{D}$ be $\delta_D>0$, i.e., $\mathbb{E}_{\tau\sim D}|y(\tau)-\hat{y}_\phi(\tau)|\le \delta_D$, 
and let the Lipschitz constant of  $\hat{y}_\phi(\tau)$ be $k\sqrt{T}R_m$.
The gap between the planner's score $L(\theta)$ and its true performance $J(\theta)$ can be bounded by:
\begin{align}
   |J(\theta)-L(\theta)|\le \delta_D +(1+k)\sqrt{T}R_m\mathbb{E}_{y\sim p_D(y)}\bigg[\underbrace{W_1(p_\theta(\tau|y^*), p_\theta(\tau|y))}_{\text{ Generation sensitivity to $y$}} + \underbrace{W_1(p_\theta(\tau|y), p_D(\tau|y))}_{\text{Imitation error on $\mathcal{D}$}}\bigg],\notag
\end{align}
where $W_1$ denotes the 1-Wasserstein distance.
\end{theorem}
Note that $\delta_D$ could be regulated to a small value via supervised training of the evaluator, and $k$ depends on the Lipschitz property of the resulting evaluator as stated before \footnote{Note that $k$ cannot approach zero without compromising $\delta_D$, as excessively small $k$ prevents the evaluator from fitting the offline dataset $\mathcal{D}$.}.
Consequently, bounding the evaluator bias in the planner's performance requires constraining the following two factors: 
\begin{itemize}
    \item the planner's generation sensitivity to condition $y$ (the first Wasserstein term)
    \item the planner's imitation error on the offline dataset (the second Wasserstein term). 
\end{itemize}
Specifically, we establish that the first Wasserstein term can be bounded by the Lipschitz constant $\text{Lip}_{W_1}(p_\theta(\tau|y))$ of the planner with respect to the condition $y$ measured by $W_1$:
\begin{align}
\label{equ:lipschiz_planner_bound}
     \mathbb{E}_{y\sim p_D(y)}[W_1(p_\theta(\tau|y^*), p_\theta(\tau|y))]\le  (1+\epsilon)y_m\text{Lip}_{W_1}(p_\theta(\tau|y)).
\end{align}
The proof is given in Appendix \ref{app:proof_lipschitz_planner}. Therefore, to ensure the boundedness of the first Wasserstein term, we constrain the planner's Lipschitz constant $\text{Lip}_{W_1}(p_\theta(\tau|y))$ to a positive hyperparameter $L_p$. A lower bound analysis of $L_p$ is provided later in Eq. \ref{equ:lower_bound_Lp}. 

Moreover, we establish that a constrained KL divergence $\mathbb{E}_{y\sim p_D(y)}[D_\text{KL}(p_D(\tau|y)\|p_\theta(\tau|y))]\le \delta_K$ could bound the expectation of the second Wasserstein distance term as follows, 
where $\delta_K>0$ is a hyperparameter and can be set to a small value, close to zero. 
See Appendix \ref{app:proof_kl_bound} for the proof.
Note that the KL divergence constraint here inherently makes the planner perform conditional behavior cloning on the offline dataset $\mathcal{D}$ \citep{guo2025deepseek}.
\begin{align}
\label{equ:kl_bound}
 \mathbb{E}_{y\sim p_D(y)}[W_1(p_\theta(\tau|y), p_D(\tau|y))]\le  \sqrt{\delta_K}.
\end{align}

Collectively, to effectively perform score maximization with a small, certifiable evaluator bias, we enforce Lipschitz continuity of the planner with respect to the condition $y$, while preserving its behavior cloning fidelity to the offline dataset $\mathcal{D}$. 
Formally, Eq. \ref{equ:max_R} is transformed to:
\begin{align}
&\max_\theta \;\;\;\; L(\theta)\qquad\qquad\qquad\qquad\qquad\qquad\qquad \;\;\;\text{(Score Maximization)}\label{equ:max_R_constrained}\\
&\;\mathrm{s.t.}\;\;\;\;\; \mathbb{E}_{y\sim p_D(y)}[D_\text{KL}(p_D(\tau|y) \| p_\theta(\tau|y))]\le \delta_K\;\;\text{(KL Constraint)}\tag{\ref{equ:max_R_constrained}{a}}\label{equ:max_R_constraint_KL}\\
&\quad\;\;\;\;\;\;\;\;\text{Lip}_{W_1}(p_\theta(\tau|y))\le L_p \qquad\qquad\qquad\qquad\text{(Lipschitz Constraint)}\tag{\ref{equ:max_R_constrained}{b}}\label{equ:max_R_constraint_Lipshcitz}
\end{align}
Eq. \ref{equ:max_R_constrained} forms the score maximization objective in AIGB-Pearl.

\begin{wrapfigure}{r}{0.5\textwidth}
\vspace{-4mm}
  \centering
  \begin{minipage}{1.0\linewidth}
    \includegraphics[width=\linewidth]{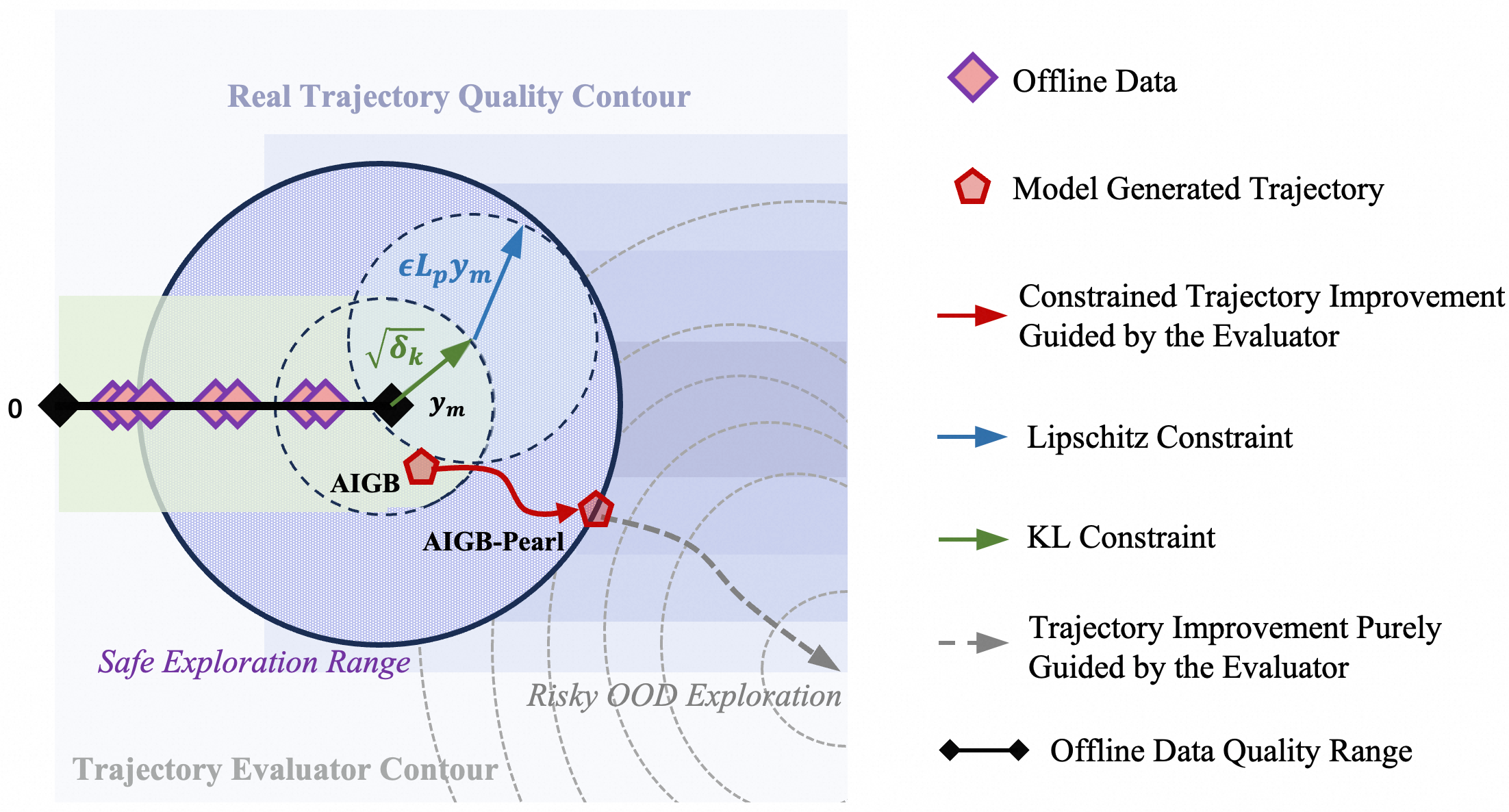}
   \caption{{\small Schematic illustration of the KL-Lipschitz-constrained score-maximization. It enables safe trajectory improvement within a theoretically certified neighborhood of the high quality trajectories in the offline dataset, where the evaluator remains high accuracy.}}
    \label{fig:kl_lipschitz}
  \end{minipage}
  \vspace{-8pt}
  
\end{wrapfigure}
\textbf{Remark 1.}
\emph{Intuitively, the KL and Lipschitz constraints jointly ensure the planner's generation under condition $y^*$ remains within a certified neighborhood of the high-quality trajectories in the offline dataset $\mathcal{D}$.
As illustrated in Fig. \ref{fig:kl_lipschitz}, the green region represents the feasible set of trajectories $p_\theta(\tau|y)$ satisfying the KL constraint for $y\in \mathcal{D}$ \footnote{Note that the KL-constraint feasible region may not be a regular topology since the generated trajectories $p_\theta(\tau|y)$ should also satisfy the Lipschitz constraint under condition $y\in \mathcal{D}$.}.
The Lipschitz constraint makes the generated trajectories $p_\theta(\tau|y^*)$ remain within a neighborhood of the best-quality trajectory in the offline dataset $p_\theta(\tau | y_m)$, as illustrated by the blue circle with radius $\epsilon L_p y_m$.
The union of all such circles constitutes the total trajectory exploration range.
Meanwhile, the evaluator trained on $\mathcal{D}$ maintains high accuracy within this $\mathcal{D}$-proximal region, and the Lipschitz regularization on the evaluator bounds its sensitivity to input perturbations, preventing drastic value fluctuations in OOD regions and promoting more reliable extrapolation.
Overall, KL-Lipschitz constrained optimization, guided by a Lipschitz continuity evaluator, enables safe trajectory improvement within a theoretically certified neighborhood of high-quality offline trajectories, effectively mitigating risky OOD exploration.
}

\subsubsection{sub-optimality gap bound}
\label{sec:sub_optimality_bound}
This section focuses on presenting and analyzing the sub-optimality bound of the solution to the proposed Eq. \ref{equ:max_R_constrained}.
Specifically, denote the solution to the true performance $J(\theta)$ as $\theta^*\triangleq\arg\max_\theta J(\theta)$, and denote the solution to the proposed Eq. \ref{equ:max_R_constrained} as $\hat{\theta}$.
The following theorem gives the sub-optimality bound of the planner's performance, and the proof is given in Appendix \ref{app:proof_suboptimality}.
\begin{theorem}[Sub-optimality Gap Bound] 
\label{thm:sub_optimality}
Let $\delta_M\triangleq \mathbb{E}_{y\sim p_D(y)}[D_{KL}(p_D(\tau|y)\|p_{\theta^*}(\tau|y^*))]$ be the expected distance between the optimal trajectory distribution and the trajectory distribution of the offline dataset $\mathcal{D}$.
The true performance gap between the optimal parameter $\theta^*$ and the solution $\hat{\theta}$ to Eq. \ref{equ:max_R_constrained}  is bounded by:
\begin{align}
    J(\theta^*) - J(\hat{\theta})\le 2\delta_D+(1+2k)\sqrt{T}R_m\bigg[\sqrt{\delta_M}+\sqrt{\delta_K}+(1+\epsilon)y_mL_p\bigg].
\end{align}
\end{theorem}
\textbf{Theoretical Result Analysis.} 
In Theorem \ref{thm:sub_optimality}, the constants $\delta_M, R_m, T, \epsilon, y_m$ characterize domain-specific properties of the auto-bidding task and the offline dataset $\mathcal{D}$. Nonetheless, a lower training error $\delta_D$ of the evaluator and a closer $k$ to $1$ correspond to a smaller sub-optimality gap. Note that $k$ cannot be smaller than $1$ without compromising $\delta_D$, as excessively small $k$ prevents the evaluator from fitting the offline dataset $\mathcal{D}$.

Moreover, in Theorem \ref{thm:sub_optimality}, a lower behavior cloning error $\delta_K$ and a lower Lipschitz constant $L_p$ of the planner lead to a smaller sub-optimality gap. However, an excessively small $L_p$ prevents the planner from behavior cloning the offline dataset $\mathcal{D}$ (as required by the KL constraint), resulting in a large $\delta_K$.
Actually, a theoretical lower bound for $L_p$ is given by the Lipschitz constant of the conditional trajectory distribution of the offline dataset $p_D(\tau|y)$:
\begin{align}
\label{equ:lower_bound_Lp}
    L_p\ge \sup_{\substack{y_1\neq y_2}}\frac{W_1(p_D(\tau|y_1), p_D(\tau|y_2))}{|y_1-y_2|}.
\end{align}
where $y_1,y_2\in\mathcal{D}$. Consequently, we leverage this lower bound of $L_p$ in AIGB-Pearl.
A further discussion on the theoretical performance range of AIGB-Pearl is given in Appendix \ref{app:theory_discussion}.


\subsection{Practical Algorithm Design}
\label{sec:practical_algorithm}
This section focuses on the practical algorithm implementation of Eq. \ref{equ:max_R_constrained}. 
Section \ref{sec:evaluator_design} first presents our reliability-enhanced evaluator architecture, followed by the synchronous-coupling-based Lipschitz planner design in Section \ref{sec:planner_loss}.

\subsubsection{Lipschitz Trajectory Evaluator}
\label{sec:evaluator_design}
As shown in Fig. \ref{fig:aigb_pearl}, the evaluator processes the trajectory $\tau$ to predict a score $\hat{y}_\phi(\tau)$ for quality estimation. 
The evaluator is trained via supervised learning using the offline dataset $\mathcal{D}$. Besides, to satisfy $\sqrt{T}R_m$-Lipschitz constraint requirement according to Theorem \ref{thm:performance_gap_bound}, we add Lipschitz regularization term to the training loss of the evaluator, which can be expressed as:
\begin{align}
\label{equ:evaluator_loss}
    l_e(\phi)= \underbrace{\mathbb{E}_{\tau\sim \mathcal{D}}\bigg[(\hat{y}_{\phi}(\tau)-y(\tau))^2\bigg]}_{\text{fitting the ground truth}}+\beta_1\underbrace{\mathbb{E}_{\tau_1,\tau_2}\bigg[|\hat{y}_\phi(\tau_1)-\hat{y}_\phi(\tau_2)|-\sqrt{T}R_m\left\|\tau_1-\tau_2\right\|_F\bigg]_+}_{\text{Lipschitz penalty}},
\end{align}
where $\beta_1>0$ is a hyper-parameter, $[\cdot]_+\triangleq\max\{0, \cdot\}$. 
Moreover, to further enhance the prediction accuracy of the evaluator, 
we devise two specific techniques, including the LLM Embedding enhancement and pair-wise learning, whose details are given in Appendix \ref{app:aigb_pearl_evaluator}.



\subsubsection{Lipschitz Planner With Synchronous Coupling}
\label{sec:planner_loss}
As shown in Fig. \ref{fig:aigb_pearl}, the planner is implemented by a Causal Transformer \citep{chen2021decision} that generates trajectories in an auto-regressive manner. 
Specifically, the model takes the condition $y$ and history states $s_{1:t}$ as input tokens, and predicts the next state as a Gaussian distribution, $p_\theta(s_{t+1}|s_{1:t},y)=\mathcal{N}(\mu_\theta(s_{1:t}, y,t), \sigma_\theta^2(s_{1:t}, y,t))$, where $\mu_\theta$ denotes the mean and $\sigma_\theta>0$ the standard deviation.
During the auto-regressive generation process, each output state is sampled from the Gaussian distribution using the reparameterization trick, i.e., $s_{t+1}=\mu_\theta(s_{1:t}, y,t)+\sigma_\theta(s_{1:t}, y,t) \cdot \eta_t$, where $\eta_t\sim \mathcal{N}(0, I)$ \footnote{Note that the time step $t$ and the static advertiser feature $x$ in the state $s_t=[t, c_{t-1}, x]$ do not need to be generated. We only generate the next cost ratio $c_t$ in practice.}.

\textbf{Regularized Planner Training Loss.}
To perform the score maximization in Eq. \ref{equ:max_R_constrained}, we involve two regularization terms in the planner's training loss $l_p(\theta)$, including a conditional behavior cloning loss, corresponding to the KL constraint Eq. \ref{equ:max_R_constraint_KL}, and a Lipschitz penalty loss, corresponding to the Lipschitz constraint Eq. \ref{equ:max_R_constraint_Lipshcitz}, i.e.,
\begin{align}
\label{equ:planner_loss}
    l_p(\theta) =& -\underbrace{\mathbb{E}_{\tau\sim p_\theta(\tau|y^*)}[\hat{y}_\phi(\tau)]}_{\text{planner score }L(\theta)} - \beta_2\underbrace{\mathbb{E}_{(\tau,y)\sim p_D}[\log p_\theta(\tau|y)]}_{\text{conditional behavior clone}} \notag\\
    &+\beta_3 \underbrace{\mathbb{E}_{y_1,y_2\in \mathcal{D}\cup\{y^*\}}\bigg[{W}_1(p_\theta(\tau|y_1),p_\theta(\tau|y_2))-L_p|y_1-y_2|\bigg]_+}_{\text{Lipschitz penalty, where  ${W}_1(p_\theta(\tau|y_1),p_\theta(\tau|y_2))$ is replaced by $ \hat{W}_1(y_1,y_2;\theta)$}},
\end{align}
where $\beta_2,\beta_3>0$ are two hyper-parameters.
With prior RL works \citep{sutton1999policy}, we can derive the closed-form expression of planner's score gradient $\nabla_\theta L(\theta)$ as shown in Appendix \ref{app:proof_score_gradient}.
The core of the planner loss lies in the computation of ${W}_1(p_\theta(\tau|y_1),p_\theta(\tau|y_2))$.

\begin{figure}
  \centering
\includegraphics[width=0.95\linewidth]{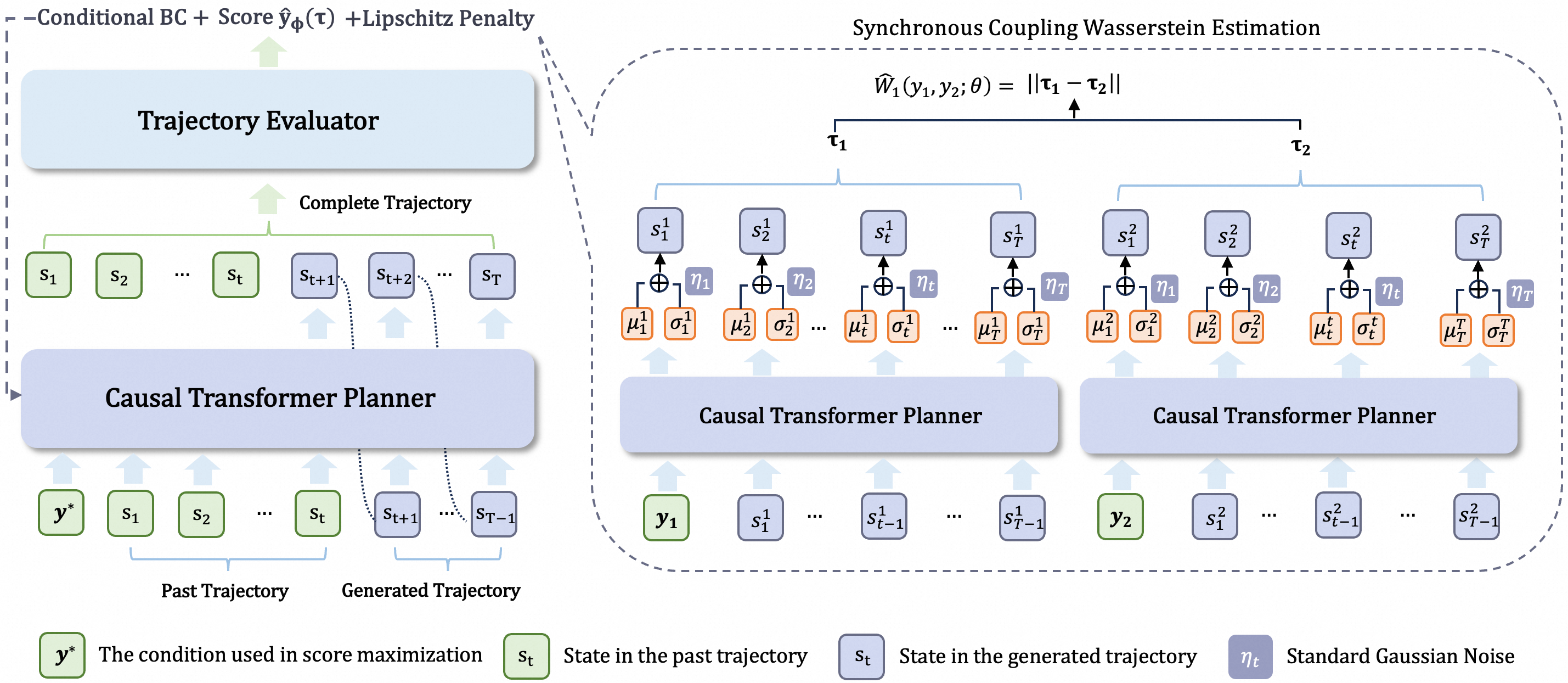}
  \caption{\textbf{AIGB-Pearl} \emph{(Planning with EvaluAtor via RL)} constructs a trajectory evaluator to score the trajectory quality and let the planner maximize the obtained score under the KL-Lipschitz constraint through continuous interaction with the evaluator. A synchronous coupling method is used to estimate the Wasserstein term in the Lipschitz penalty.
  }
  \label{fig:aigb_pearl}
  \vspace{-0.5cm}
\end{figure}

\textbf{Wasserstein Upper Bound as Surrogate.}
Accurate computation of this Wasserstein distance term is challenging, as it requires finding the optimal coupling between $p_\theta(\tau|y_1)$ and $p_\theta(\tau|y_2)$ that minimizes the expected transportation cost.
Nonetheless, we can choose a certain coupling $\gamma\in \Gamma(p_\theta(\tau|y_1), p_\theta(\tau | y_2))$ to obtain an upper bound of this Wasserstein term, i.e.,  
\begin{align}
\label{equ:upper_bound}
{W}_1(p_\theta(\tau|y_1),p_\theta(\tau|y_2))&\triangleq \inf_{\gamma\in \Gamma(p_\theta(\tau|y_1), p_\theta(\tau|y_2))}\mathbb{E}_{(\tau_1,\tau_2)\sim\gamma}\bigg[\sum_{t}\| s_t^1-s_t^2\|\bigg] \notag\\
    &\le \mathbb{E}_{\eta_{1:T}\sim \mathcal{N}(0,I)}\bigg[\sum_t\|s_t^1-s_t^2]\|\bigg]\triangleq \hat{W}_1(y_1,y_2;\theta).
\end{align}
where $\hat{W}_1(y_1,y_2;\theta)$ denotes the upper bound, and $s^i_t$ is the $t$-th state in trajectory $\tau_i$. 
It can be seen that $\hat{W}_1(y_1,y_2;\theta)\le L_p|y_1-y_2|$ acts as a sufficient condition to make the planner $L_p$-Lipschitz continuous. Thus, we replace ${W}_1(p_\theta(\tau|y_1),p_\theta(\tau|y_2))$ by this upper bound 
in the planner loss.

\textbf{Synchronous Coupling Wasserstein.}
Instead of using random couplings, we employ a \emph{synchronous coupling} $\gamma_\text{sync}$ to make the upper bound tighter.
Specifically, two trajectories $\tau_1$ and $\tau_2$---conditioned on $y_1$ and $y_2$, respectively---are generated using the same sequence of Gaussian noise $\{\eta_1, \eta_2, ...\eta_T\}$.
The definition of $\hat{W}_1(y_1,y_2;\theta)$ is given in Eq. \ref{equ:upper_bound}.
Compared to random couplings, the synchronous coupling reduces spurious variance in the trajectory comparison by aligning stochasticity through shared noise, resulting in a tighter upper bound on the Wasserstein distance \citep{lindvall2002lectures}.

Moreover, if we make the predicted variance $\sigma_\theta$ of the planner a fixed constant, then the expression of  $\hat{W}_1(y_1,y_2;\theta)$ can be further simplified to 
$\hat{W}_1(y_1,y_2;\theta)=\sum_t\|\mu_\theta(s^1_{1:t},y_1,t)-\mu_\theta(s^2_{1:t},y_2, t) \|$.
The overall AIGB-Pearl algorithm is summarized in Algorithm \ref{algo} in Appendix \ref{app:aigb_pearl} due to page limits.

\vspace{-1.5mm}

\setlength{\tabcolsep}{3pt}
\setlength{\abovecaptionskip}{-0.1pt}
\begin{table*}[t]
    \centering
    \caption{Overall performance (GMV) in simulated experiments with 30 advertisers.
 $\Delta$ indicates the relative improvement of AIGB-Pearl against the most competitive baseline (which is underlined). Note that the absolute values are normalized without specific meanings; only $\Delta$ matters.}
 \footnotesize
    \begin{tabular}{c|cccccccc|cc}
    \toprule
    \textbf{Budget} & \textbf{USCB} & \textbf{BCQ} &\textbf{CQL} & \textbf{IQL} & \textbf{MORL} & 
    \textbf{MOPO} & \textbf{DT} & \textbf{DiffBid} & \textbf{AIGB-Pearl} & $\bm{\Delta}$\\
    \midrule
       1.5k & {454.25} & 454.72 & 461.82 & 456.80 & 468.49 & 470.38&477.39 & \underline{480.76} & {502.98} & \textbf{+4.62\%} \\
  2.0k & 482.67 & 483.50 & 475.78 & 486.56 & 488.12 & 489.27 &507.30 & \underline{511.17} & {521.84} & \textbf{+2.09\%} \\
    2.5k & 497.66 & 498.77 & 481.37 & 518.27 & 511.93 & 523.91 & 527.88 & \underline{531.29} & {545.03} & \textbf{+2.59\%} \\
    3.0k & 500.60 & 501.86 & 491.36 & 549.19 & 553.91 & 549.01 & 550.66 & \underline{556.32} & {574.17} & \textbf{+3.21\%} \\
    \bottomrule
    \end{tabular}
    \label{tab:main_exp_sim}
    \vspace{-5pt}
\end{table*}

\begin{table*}
    \centering
    \caption{Overall performance in real-world A/B tests, involving 6k advertisers over 19 days.}
    \begin{adjustbox}{max width =1.0\linewidth}
    \begin{tabular}{c|cccc|c|cccc}
    \toprule
       Methods & GMV & BuyCnt & ROI & Cost & Methods & GMV & BuyCnt & ROI & Cost \\
    
    \midrule
       \textbf{DiffBid} & 76,390,174 & 650,962 & 5.31 & 14,395,290 & \textbf{USCB} & 52,182,805 & 516,994 & 4.92 & 10,598,486\\
       \textbf{AIGB-Pearl} & {78,676,009} & {665,173} & {5.41} & 14,551,054 & \textbf{AIGB-Pearl} & {53,973,101} & {520,796} & {5.13} & 10,515,772\\
       $\bm \Delta$ & \textbf{+3.00\%} & \textbf{+2.20\%} & \textbf{+1.89\%} & +1.10\% & $\bm \Delta$ & \textbf{+3.43\%} & \textbf{+0.74\%} & \textbf{+4.24\%} & -0.78\% \\
    \midrule
    
    Methods & GMV & BuyCnt & ROI & Cost & Methods & GMV & BuyCnt & ROI & Cost \\
    \midrule
        \textbf{DT} & 34,808,665 & 341,995 & 5.61 & 6,205,665 & \textbf{MOPO} & 51,674,071 & 579,332 & 3.08 & 16,771,892 \\
       \textbf{AIGB-Pearl} & {35,957,933} & {344,194} & {5.77} & 6,246,512 & \textbf{AIGB-Pearl} & {53,292,945} & {591,741} & {3.23} & 16,475,670 \\
       $\bm \Delta$ & \textbf{+3.30\%} & \textbf{+0.64\%} & \textbf{+0.16\%} & +0.66\% & $\bm \Delta$ & \textbf{+3.13\%} & \textbf{+2.14\%} & \textbf{+4.87\%} & -1.77\%\\
    \bottomrule
    \end{tabular}
    \end{adjustbox}
    \label{tab:main_exp_real}
    \vspace{-15pt}
\end{table*}

\section{Experiments}
\vspace{-2mm}
We conduct both simulated and real-world experiments to validate the effectiveness of our approach.
In the experiments, we mainly investigate the following {Research Questions} (\textbf{RQ}s): (1) Does enhancing AIGB with policy optimization improve overall performance, and can it generalize better to unseen data compared to existing AIGB methods? (Section \ref{sec:exp_overall_performance})
(2) How does the KL-Lipschitz constraint affect the performance of the planner? (Section \ref{sec:ablation_study})
(3) Can the proposed method guarantee the Lipschitz property of the evaluator and the planner? (Section \ref{sec:exp_evaluator_accuracy}).
{(4) What is the evaluator’s accuracy on the training data, and how well does it generalize to unseen data? (Section \ref{sec:exp_evaluator_acc_examination}).} The training stability of AIGB-Pearl is studied in Appendix \ref{app:exp_stability}.

\vspace{-1.5mm}

\subsection{Experiment Setup}
\vspace{-1.5mm}

\textbf{Experiment Environment.}
We conduct simulated experiments in an open-source offline advertising system with 30 advertisers of four budget levels (1.5k, 2.0k, 2.5k, and 3.0k),
as in \citep{sorl, aigb}. 
The offline dataset comprises 5k trajectories generated by 20 advertisers.
Extra detailed settings of simulated experiments are given in Appendix  \ref{app:exp_simulate}. 
For real-world experiments, we conduct online A/B tests on one of the world's largest E-commerce platforms, TaoBao.
The offline dataset comprises 200k trajectories of 10k advertisers.
See Appendix \ref{app:exp_real_world} for extra detailed settings of real-world experiments. 
In both simulated and real-world experiments, we employ the same inverse dynamics model from \citet{agrawal2016learning} as the controller in AIGB.
{Moreover, the evaluator is trained on the entire offline dataset, and its generalization ability is evaluated using $K$-fold cross-validation with $K=5$.}

\textbf{Baselines.}
We compare our method with state-of-the-art AIGB methods, including \textbf{DiffBid} \citep{aigb} and \textbf{DT} \citep{chen2021decision}, which learn from conditional behavior cloning of offline datasets using a diffusion model and a Causal Transformer, respectively.
We also compare our method with RL auto-bidding methods, including  \textbf{USCB} \citep{uscb} that learns the auto-bidding policy in a manually constructed advertising system with DDPG \citep{silver2014deterministic}; and offline RL auto-bidding methods, including model-free offline RL methods \textbf{BCQ} \citep{fujimoto2019off}, \textbf{CQL} \citep{kumar2020conservative}, and \textbf{IQL} \citep{kostrikovoffline}, and model-based offline RL methods \textbf{MOPO} \citep{yu2020mopo} and \textbf{MORL} \citep{mou2025permutation}.

\vspace{-1mm}
\textbf{Performance Index.}
The objective in the auto-bidding problem Eq. \ref{equ:prob_formulation}, i.e., the cumulative rewards over the bidding episode, acts as the main performance index in our experiments and is referred to as the \textit{gross merchandise volume}, \textbf{GMV}.
In addition, we utilize three other metrics commonly used in the auto-bidding problem to evaluate the performance of our approach. The first metric is the total number of impressions won over the bidding episode, referred to as the \textbf{BuyCnt}. 
The second metric is the \textbf{Cost} over the bidding episode, and the third is the \emph{return on investment} \textbf{ROI}, defined as the ratio of GMV to Cost. 
Note that larger values of GMV, BuyCnt, and ROI with a Cost oscillating within an acceptable tolerance ($\pm 2\%$) indicate a better performance.  

\vspace{-1mm}
\subsection{Overall Performance}
\label{sec:exp_overall_performance}

\vspace{-1mm}

\begin{table*}
    \centering
    \caption{Generalization performance in real-world A/B tests with unseen advertisers against AIGB methods, involving 4k advertisers over 19 days.}
    \begin{adjustbox}{max width =1.0\linewidth}
    \footnotesize
    \begin{tabular}{c|cccc|c|cccc}
    \toprule
       Methods & GMV & BuyCnt & ROI & Cost & Methods & GMV & BuyCnt & ROI & Cost \\
    
    \midrule
       \textbf{DiffBid} & 67,092,973 & 553,020 & 5.39 & 12,444,306 & \textbf{DT} & 30,562,007 &300,271  &  5.61&5,450,573 \\
       \textbf{AIGB-Pearl} & {69,252,539} & {565,776}  & {5.53} & 12,534,379 & \textbf{AIGB-Pearl} & {31,502,309}  & {305,202}  &  {5.74} & 5,484,473 \\
       $\bm \Delta$ & \textbf{+3.32\%} & \textbf{+2.31\%} & \textbf{+2.48\%} & +0.72\% & $\bm \Delta$ & \textbf{+3.08\%}  & \textbf{+1.64\%}  & \textbf{+2.32\%} & +0.62\% \\
        \bottomrule
    \end{tabular}
    \end{adjustbox}
    \label{tab:gene_exp_real}
\end{table*}

\textbf{To answer RQ(1): } 
Table \ref{tab:main_exp_sim} shows that our method consistently outperforms all baselines in GMV across all four budget levels in simulated experiments.
In real-world experiments, Table \ref{tab:main_exp_real} shows that our method also achieves superior performance on GMV, BuyCnt, and ROI, with Cost fluctuations of less than 2\%.
Notably, both simulated and real-world experiments consistently
demonstrate that AIGB-Pearl achieves a \textbf{+3\%} improvement in GMV over the AIGB, the state-of-the-art auto-bidding method.
Since our method and DiffBid share the same controller, the performance gain stems solely from the planner. 
This provides strong empirical evidence that the proposed conservative RL learning for score maximization effectively enhances overall performance.

Notably, we also apply AIGB-Pearl to another important auto-bidding problem, named TargetROAS (See Appendix \ref{app:exp_targetroas}). Real-world experiments show that AIGB-Pearl achieves a \textbf{+5\%} improvement in GMV compared to AIGB.
{It is worth noting that a GMV uplift exceeding 2\% is highly significant, translating to \textbf{millions} of RMB in additional \textbf{daily} GMV on Taobao-scale advertising platforms.}

\textbf{Generalization Ability.}
We evaluate AIGB-Pearl on advertisers not used to generate trajectories in the offline dataset and compare it with existing AIGB methods. 
For simplicity, we refer to these advertisers as \emph{advertisers outside the offline dataset}. 
Table \ref{tab:gene_exp_real} reports the performance on 4k advertisers outside the offline dataset in real-world experiments. 
AIGB-Pearl consistently delivers better results in GMV (\textbf{+3\%}), BuyCnt, and ROI, while maintaining Cost fluctuations within 2\% compared to the baselines.
This indicates that the proposed method has better generalization ability than AIGB. 

\vspace{-1mm}
\subsection{Ablation Study}
\label{sec:ablation_study}
\vspace{-1mm}

\textbf{To answer RQ(2):}
We remove the KL and Lipschitz constraints from AIGB-Pearl individually and evaluate the model's performance in each ablated variant using real-world A/B tests. The results are presented in Table \ref{tab:ablation_study}.
It can be seen that the KL constraint contributes \textbf{+1.1\%} improvement in GMV, and the Lipschitz constraint provides \textbf{+1.8\%} improvement in GMV, demonstrating their respective roles in enhancing AIGB-Pearl’s performance.

\begin{table*}[t]
    \centering
    \caption{Ablation Study. The effectiveness of the KL constraint and the Lipschitz constraint in
    Real-world A/B tests, involving 6k advertisers over 8 days. }
    \begin{adjustbox}{max width =1.0\linewidth}
    \footnotesize
    \begin{tabular}{c|cccc|c|cccc}
    \toprule
       AIGB-Pearl & GMV & BuyCnt & ROI & Cost & Methods & GMV & BuyCnt & ROI & Cost \\
    
    \midrule
       \textbf{w/o KL} & 30,906,963  & 292,605 & 4.25 & 7,269,018 & \textbf{w/o Lipschitz} & 32,284,972 & 268,551 & 5.73 & 5,634,304 \\
       \textbf{with KL} & {31,243,688} & {292,783}  & {4.26} & 7,342,485 & \textbf{with Lipschitz} & {32,869,329} & {281,979} & {5.79} & 5,678,252\\
       $\bm \Delta$ & \textbf{+1.09\%} & \textbf{+0.06\%} & \textbf{+0.08\%} & +1.01\% 
       & $\bm \Delta$ & \textbf{+1.81\%} & \textbf{+0.50\%} & \textbf{+1.05\%} & +0.78\% \\
        \bottomrule
    \end{tabular}
    \end{adjustbox}
    \label{tab:ablation_study}
    \vspace{-10pt}
\end{table*}

\textbf{Visualization.}
Three AIGB-Pearl-generated trajectory examples are presented in 
Fig. \ref{fig:visualization}.
As can be observed, the trajectories generated by AIGB-Pearl are plausible.
In contrast, the ablated variant without the KL and Lipschitz constraints produces trajectories that deviate significantly from the optimal trajectory in the offline dataset 
and exhibit clear pathological behaviors---such as excessive budget consumption, backward-trending pacing, and under-utilization of available budgets (see Appendix \ref{app:pathological_behaviors} for explanation)---which further validate the KL-Lipschitz constraint necessity. 

\begin{figure}[t]
  \centering
\includegraphics[width=1.0\linewidth]{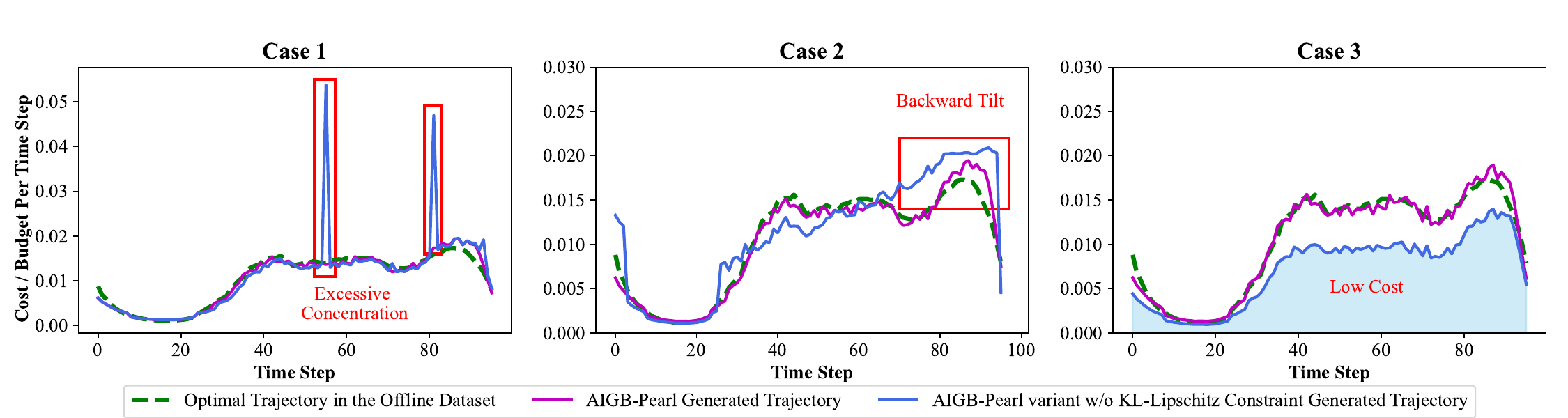}
  \caption{Trajectory Generation Visualization. Three cases are presented. Here, the  
  AIGB-Pearl generates plausible trajectories, whereas its variant without the KL-Lipschitz constraint produces generations that significantly deviate from the reference and exhibit evident issues.}
  \label{fig:visualization}
  \vspace{-1mm}
\end{figure}

\vspace{-1mm}
\setlength{\tabcolsep}{10pt}
\begin{table*}[t]
    \centering
    \caption{{Evaluator accuracy for simulated and real-world experiments. Results are reported for training data and OOD data evaluated using 5-fold cross-validation.}}
    \begin{adjustbox}{max width =1.0\linewidth}
    \begin{tabular}{ccc|ccc}
    \toprule
        \makecell[c]{\textbf{Simulated Exp}} & Training Data & OOD Data (Cross-Validation) & \makecell[c]{\textbf{Real-world Exp}} & Training Data & OOD Data (Cross-Validation) \\
    
    \midrule
       MAE $\downarrow$ & 0.6  & 0.7 $\pm$ 0.06 & MAE $\downarrow$& 1.0 & 1.2 $\pm$ 0.03\\
       AUC $\uparrow$ & 89.9\%  & 85.5\% $\pm$ 0.5\% & AUC $\uparrow$& 77.4\% & 75.1\% $\pm$ 0.2\%\\
        \bottomrule
    \end{tabular}
    \end{adjustbox}
    \label{tab:evaluator_accuracy}
    \vspace{-5pt}
    \vspace{-2mm}
\end{table*}

\vspace{-2mm}
\subsection{Lipschitz Value Examination}
\label{sec:exp_evaluator_accuracy}
\vspace{-2mm}

\begin{wrapfigure}{r}{0.5\textwidth}
\vspace{-4mm}
  \centering
  \begin{minipage}{0.48\linewidth}
    \includegraphics[width=\linewidth]{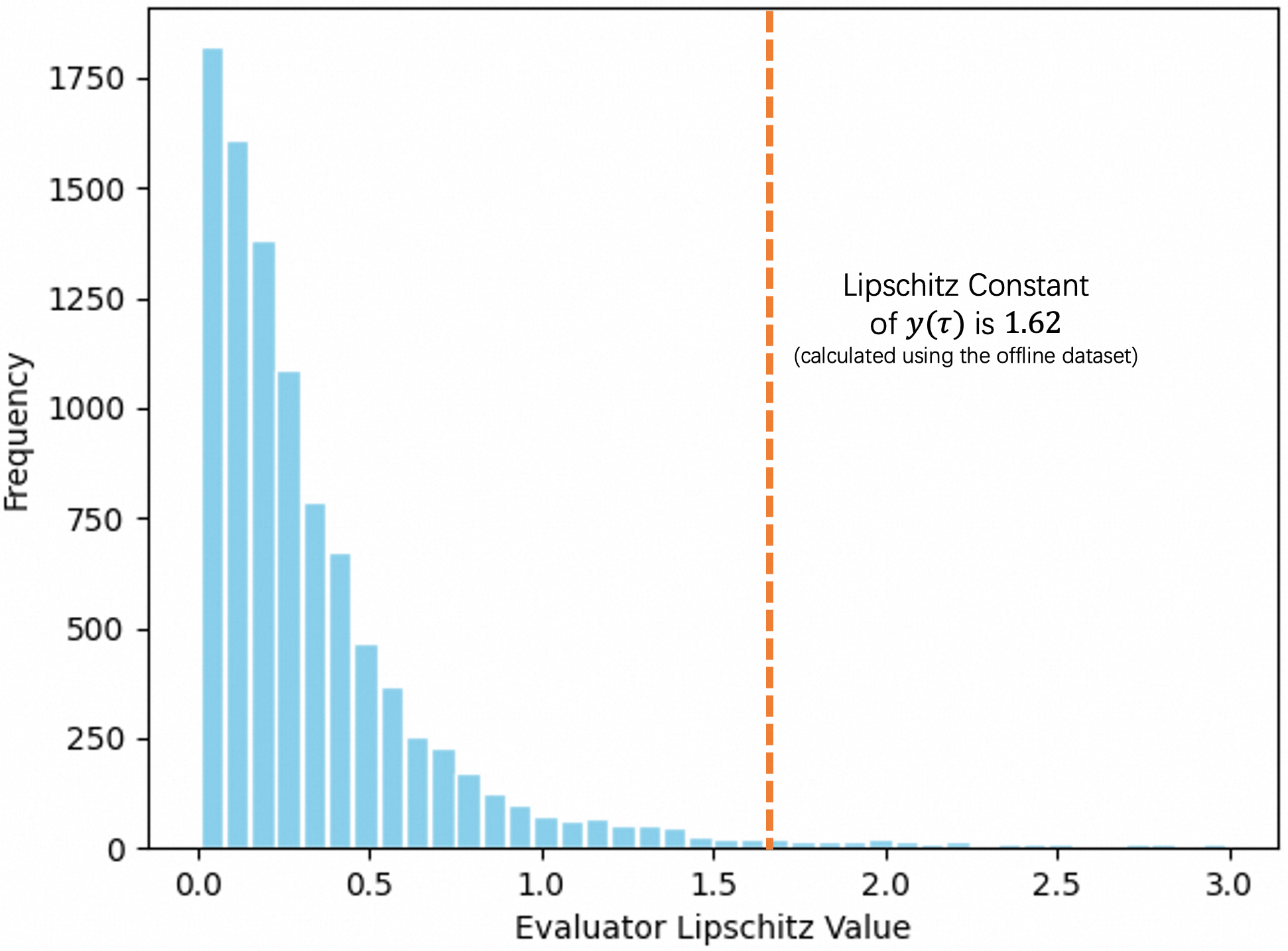}
    \caption{Examination of Evaluator Lipschitz.}
     \label{fig:evaluator_lipschitz}
  \end{minipage}
  \hfill
  \begin{minipage}{0.48\linewidth}
    \includegraphics[width=\linewidth]{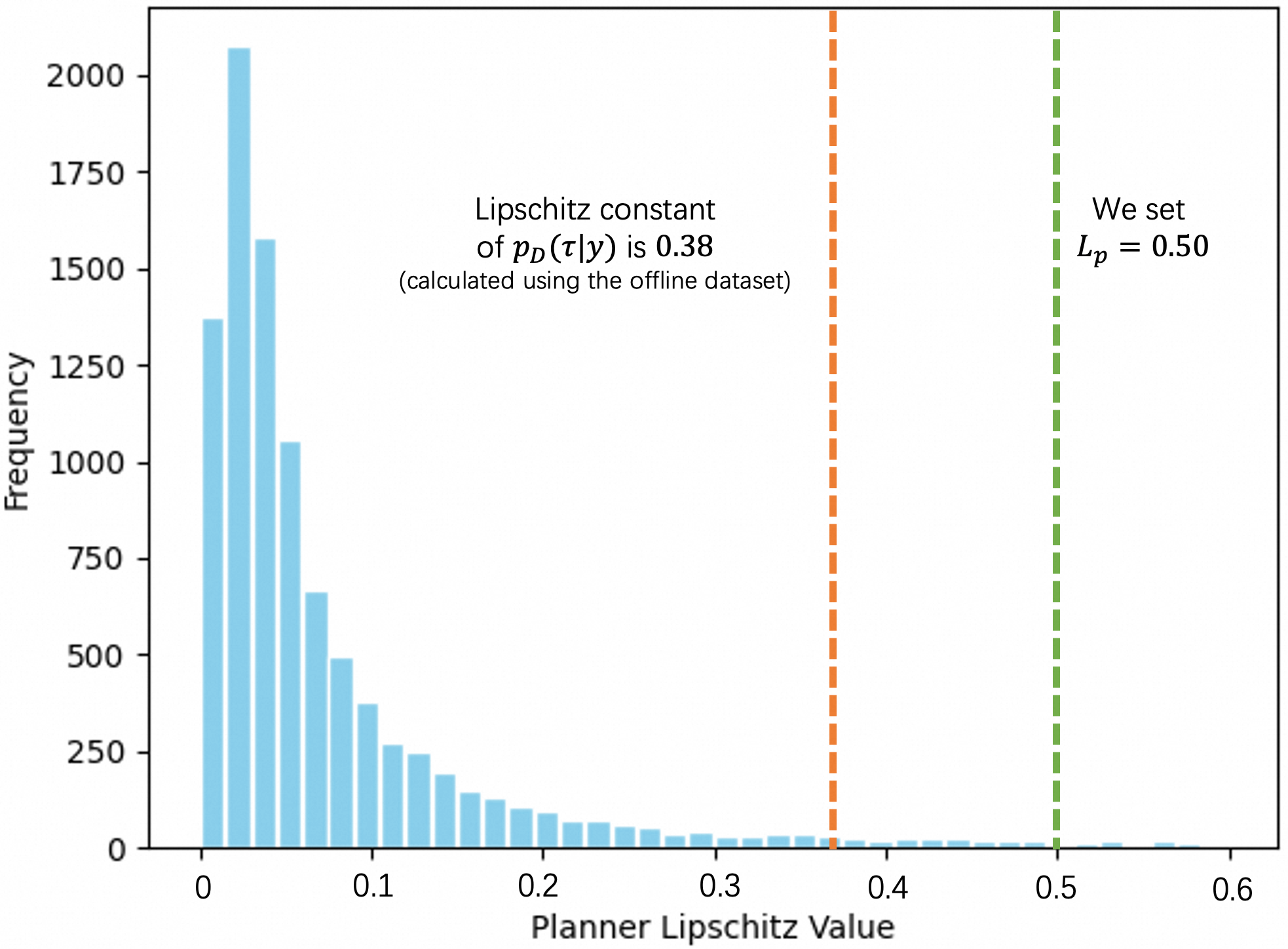}
    \caption{Examination of Planner Lipschitz.}
    \label{fig:planner_lipschitz}
  \end{minipage}
  \label{fig:sidebyside}
  \vspace{-10pt}
\end{wrapfigure}
\textbf{To answer RQ(3):}
We report that the Lipschitz value of the trajectory quality $y(\tau)$ and the conditional trajectory distribution $p_D$ of the offline dataset are $1.62$ and $0.38$, respectively. We set $L_p=0.50$, which is close to its lower bound estimate of $0.38$
\footnote{As the data-driven lower bound estimation can be underestimated, we slightly increase its value in practice.}.
To calculate the Lipschitz constants of the evaluator and the planner, we sample $8,000$ pairs of trajectories and compute their Lipschitz constants. The results are shown in Fig. \ref{fig:evaluator_lipschitz} and Fig. \ref{fig:planner_lipschitz}.  
It can be observed that most sample values satisfy the Lipschitz constraint, and the Lipschitz constants of the evaluator $\hat{y}_\phi(\tau)$ and planner $p_\theta(\tau|y)$ are $2.2$ and $0.56$, respectively, near $1.62$ and $0.50$. 
This indicates that the models' Lipschitz constraints are successfully satisfied.


\vspace{-3mm}
\subsection{Evaluator Accuracy Examination}
\label{sec:exp_evaluator_acc_examination}
\vspace{-2mm}

{\textbf{Accuracy Metrics.} The evaluator's accuracy is assessed along two dimensions, including the \emph{absolute accuracy} measured by mean absolute error (MAE) metrics, reflecting how close the predicted scores are to ground truth scores, and the \emph{ranking accuracy} by AUC metric, reflecting the correctness of relative rankings between trajectory pairs.  Note that the MAE of each advertiser’s data sample is normalized by its budget to ensure comparability across advertisers.
A lower MAE, together with a higher AUC, indicates better evaluator accuracy.}

{
\textbf{To answer RQ(4):} We report the accuracy of the trained evaluator in both simulated and real-world experiments in Table \ref{tab:evaluator_accuracy}. 
We evaluate the evaluator's accuracy on the training data and its generalization ability using $K$-fold cross-validation, where $K=5$. 
To the best of our knowledge, we are the first to introduce the trajectory evaluator into the generative auto-bidding framework.
The reasonableness of our evaluator is evidenced by its pairwise ranking accuracies of 86\% AUC and 75\% AUC on OOD trajectories in simulated and real-world experiments, respectively, substantially above the 50\% random-chance level, despite the high complexity and dynamic nature of the bidding environment. Importantly, with the guidance of the trained evaluator, the planner outperforms state-of-the-art AIGB methods even on OOD data, as demonstrated in the Table. \ref{tab:gene_exp_real}.
}


\vspace{-3mm}
\section{Conclusions}
\label{sec:conclusion}
\vspace{-3mm}

This paper proposes AIGB-Pearl, which enhances AIGB by incorporating reward evaluation and policy optimization.
By introducing a trajectory evaluator and a provably KL-Lipschitz-constrained score-maximization objective, our approach ensures safe and efficient generalization beyond the offline dataset, supported by theoretical guarantees.
Extensive simulated and real-world experiments validate the state-of-the-art performance of our approach.
\newpage


\bibliography{iclr2026_conference}

@book{villani2021topics,
  title={Topics in optimal transportation},
  author={Villani, C{\'e}dric},
  volume={58},
  year={2021},
  publisher={American Mathematical Soc.}
}

@article{mou2025permutation,
  title={Permutation Equivariant Model-based Offline Reinforcement Learning for Auto-bidding},
  author={Mou, Zhiyu and Xu, Miao and Chen, Wei and Bai, Rongquan and Yu, Chuan and Xu, Jian},
  journal={arXiv preprint arXiv:2506.17919},
  year={2025}
}

@article{mao2024doubly,
  title={Doubly mild generalization for offline reinforcement learning},
  author={Mao, Yixiu and Wang, Qi and Qu, Yun and Jiang, Yuhang and Ji, Xiangyang},
  journal={Advances in Neural Information Processing Systems},
  volume={37},
  pages={51436--51473},
  year={2024}
}

@article{mao2024offline,
  title={Offline reinforcement learning with ood state correction and ood action suppression},
  author={Mao, Yixiu and Wang, Qi and Chen, Chen and Qu, Yun and Ji, Xiangyang},
  journal={Advances in Neural Information Processing Systems},
  volume={37},
  pages={93568--93601},
  year={2024}
}

@article{mao2023supported,
  title={Supported value regularization for offline reinforcement learning},
  author={Mao, Yixiu and Zhang, Hongchang and Chen, Chen and Xu, Yi and Ji, Xiangyang},
  journal={Advances in Neural Information Processing Systems},
  volume={36},
  pages={40587--40609},
  year={2023}
}

@book{Tsybakov2008introduction,
author = {Tsybakov, Alexandre B.},
title = {Introduction to Nonparametric Estimation},
year = {2008},
isbn = {0387790519},
publisher = {Springer Publishing Company, Incorporated},
edition = {1st},
}

@inproceedings{gligorijevic2020bid,
  title={Bid shading in the brave new world of first-price auctions},
  author={Gligorijevic, Djordje and Zhou, Tian and Shetty, Bharatbhushan and Kitts, Brendan and Pan, Shengjun and Pan, Junwei and Flores, Aaron},
  booktitle={Proceedings of the 29th ACM International Conference on Information \& Knowledge Management},
  pages={2453--2460},
  year={2020}
}

@inproceedings{wu2015predicting,
  title={Predicting winning price in real time bidding with censored data},
  author={Wu, Wush Chi-Hsuan and Yeh, Mi-Yen and Chen, Ming-Syan},
  booktitle={Proceedings of the 21th ACM SIGKDD International Conference on Knowledge Discovery and Data Mining},
  pages={1305--1314},
  year={2015}
}

@article{chen2021decision,
  title={Decision transformer: Reinforcement learning via sequence modeling},
  author={Chen, Lili and Lu, Kevin and Rajeswaran, Aravind and Lee, Kimin and Grover, Aditya and Laskin, Misha and Abbeel, Pieter and Srinivas, Aravind and Mordatch, Igor},
  journal={Advances in neural information processing systems},
  volume={34},
  pages={15084--15097},
  year={2021}
}

@article{sutton1999policy,
  title={Policy gradient methods for reinforcement learning with function approximation},
  author={Sutton, Richard S and McAllester, David and Singh, Satinder and Mansour, Yishay},
  journal={Advances in neural information processing systems},
  volume={12},
  year={1999}
}

@book{lindvall2002lectures,
  title={Lectures on the coupling method},
  author={Lindvall, Torgny},
  year={2002},
  publisher={Courier Corporation}
}

@book{villani2008optimal,
  title={Optimal transport: old and new},
  author={Villani, C{\'e}dric and others},
  volume={338},
  year={2008},
  publisher={Springer}
}

@inproceedings{uscb,
author = {He, Yue and Chen, Xiujun and Wu, Di and Pan, Junwei and Tan, Qing and Yu, Chuan and Xu, Jian and Zhu, Xiaoqiang},
title = {A Unified Solution to Constrained Bidding in Online Display Advertising},
year = {2021},
booktitle = {Proceedings of the 27th ACM SIGKDD Conference on Knowledge Discovery \& Data Mining},
pages = {2993–3001},
numpages = {9},
series = {KDD '21}
}

@article{wang2022diffusion,
  title={Diffusion policies as an expressive policy class for offline reinforcement learning},
  author={Wang, Zhendong and Hunt, Jonathan J and Zhou, Mingyuan},
  journal={arXiv preprint arXiv:2208.06193},
  year={2022}
}

@article{guo2025deepseek,
  title={Deepseek-r1: Incentivizing reasoning capability in llms via reinforcement learning},
  author={Guo, Daya and Yang, Dejian and Zhang, Haowei and Song, Junxiao and Zhang, Ruoyu and Xu, Runxin and Zhu, Qihao and Ma, Shirong and Wang, Peiyi and Bi, Xiao and others},
  journal={arXiv preprint arXiv:2501.12948},
  year={2025}
}

@article{kidambi2020morel,
  title={Morel: Model-based offline reinforcement learning},
  author={Kidambi, Rahul and Rajeswaran, Aravind and Netrapalli, Praneeth and Joachims, Thorsten},
  journal={Advances in neural information processing systems},
  volume={33},
  pages={21810--21823},
  year={2020}
}

@book{sutton1998reinforcement,
  title={Reinforcement learning: An introduction},
  author={Sutton, Richard S and Barto, Andrew G and others},
  volume={1},
  year={1998},
  publisher={MIT press Cambridge}
}

@article{agrawal2016learning,
  title={Learning to poke by poking: Experiential learning of intuitive physics},
  author={Agrawal, Pulkit and Nair, Ashvin V and Abbeel, Pieter and Malik, Jitendra and Levine, Sergey},
  journal={Advances in neural information processing systems},
  volume={29},
  year={2016}
}

@inproceedings{cao2007learning,
  title={Learning to rank: from pairwise approach to listwise approach},
  author={Cao, Zhe and Qin, Tao and Liu, Tie-Yan and Tsai, Ming-Feng and Li, Hang},
  booktitle={Proceedings of the 24th international conference on Machine learning},
  pages={129--136},
  year={2007}
}

@article{chen2015recommender,
  title={Recommender systems based on user reviews: the state of the art},
  author={Chen, Li and Chen, Guanliang and Wang, Feng},
  journal={User Modeling and User-Adapted Interaction},
  volume={25},
  pages={99--154},
  year={2015},
  publisher={Springer}
}

@inproceedings{silver2014deterministic,
  title={Deterministic policy gradient algorithms},
  author={Silver, David and Lever, Guy and Heess, Nicolas and Degris, Thomas and Wierstra, Daan and Riedmiller, Martin},
  booktitle={International conference on machine learning},
  pages={387--395},
  year={2014},
  organization={Pmlr}
}

@article{sorl,
  title={Sustainable online reinforcement learning for auto-bidding},
  author={Mou, Zhiyu and Huo, Yusen and Bai, Rongquan and Xie, Mingzhou and Yu, Chuan and Xu, Jian and Zheng, Bo},
  journal={Advances in Neural Information Processing Systems},
  volume={35},
  pages={2651--2663},
  year={2022}
}

@inproceedings{aigb,
author = {Guo, Jiayan and Huo, Yusen and Zhang, Zhilin and Wang, Tianyu and Yu, Chuan and Xu, Jian and Zheng, Bo and Zhang, Yan},
title = {Generative Auto-bidding via Conditional Diffusion Modeling},
year = {2024},
booktitle = {Proceedings of the 30th ACM SIGKDD Conference on Knowledge Discovery and Data Mining},
pages = {5038–5049},
numpages = {12},
series = {KDD '24}
}

@inproceedings{sohl2015deep,
  title={Deep unsupervised learning using nonequilibrium thermodynamics},
  author={Sohl-Dickstein, Jascha and Weiss, Eric and Maheswaranathan, Niru and Ganguli, Surya},
  booktitle={International conference on machine learning},
  pages={2256--2265},
  year={2015},
  organization={pmlr}
}

@article{peng2024deadly,
  title={Deadly triad matters for offline reinforcement learning},
  author={Peng, Zhiyong and Liu, Yadong and Zhou, Zongtan},
  journal={Knowledge-Based Systems},
  volume={284},
  pages={111341},
  year={2024},
  publisher={Elsevier}
}

@article{yu2020mopo,
  title={Mopo: Model-based offline policy optimization},
  author={Yu, Tianhe and Thomas, Garrett and Yu, Lantao and Ermon, Stefano and Zou, James Y and Levine, Sergey and Finn, Chelsea and Ma, Tengyu},
  journal={Advances in Neural Information Processing Systems},
  volume={33},
  pages={14129--14142},
  year={2020}
}

@article{balseiro2021robust,
  title={Robust auction design in the auto-bidding world},
  author={Balseiro, Santiago and Deng, Yuan and Mao, Jieming and Mirrokni, Vahab and Zuo, Song},
  journal={Advances in Neural Information Processing Systems},
  volume={34},
  pages={17777--17788},
  year={2021}
}

@inproceedings{deng2021towards,
  title={Towards efficient auctions in an auto-bidding world},
  author={Deng, Yuan and Mao, Jieming and Mirrokni, Vahab and Zuo, Song},
  booktitle={Proceedings of the Web Conference 2021},
  pages={3965--3973},
  year={2021}
}

@inproceedings{wen2022cooperative,
  title={A cooperative-competitive multi-agent framework for auto-bidding in online advertising},
  author={Wen, Chao and Xu, Miao and Zhang, Zhilin and Zheng, Zhenzhe and Wang, Yuhui and Liu, Xiangyu and Rong, Yu and Xie, Dong and Tan, Xiaoyang and Yu, Chuan and others},
  booktitle={Proceedings of the Fifteenth ACM International Conference on Web Search and Data Mining},
  pages={1129--1139},
  year={2022}
}

@inproceedings{fujimoto2019off,
  title={Off-policy deep reinforcement learning without exploration},
  author={Fujimoto, Scott and Meger, David and Precup, Doina},
  booktitle={International conference on machine learning},
  pages={2052--2062},
  year={2019},
  organization={PMLR}
}

@inproceedings{kostrikovoffline,
  title={Offline Reinforcement Learning with Implicit Q-Learning},
  author={Kostrikov, Ilya and Nair, Ashvin and Levine, Sergey},
  booktitle={International Conference on Learning Representations},
  year={2022}
}

@article{kumar2020conservative,
  title={Conservative q-learning for offline reinforcement learning},
  author={Kumar, Aviral and Zhou, Aurick and Tucker, George and Levine, Sergey},
  journal={Advances in neural information processing systems},
  volume={33},
  pages={1179--1191},
  year={2020}
}

@inproceedings{Jin_2018, 
   title={Real-Time Bidding with Multi-Agent Reinforcement Learning in Display Advertising},
   booktitle={Proceedings of the 27th ACM International Conference on Information and Knowledge Management},
   author={Jin, Junqi and Song, Chengru and Li, Han and Gai, Kun and Wang, Jun and Zhang, Weinan},
   year={2018},
   pages={2193–2201},
}

@inproceedings{ajay2023is,
title={Is Conditional Generative Modeling all you need for Decision Making?},
author={Anurag Ajay and Yilun Du and Abhi Gupta and Joshua B. Tenenbaum and Tommi S. Jaakkola and Pulkit Agrawal},
booktitle={The Eleventh International Conference on Learning Representations },
year={2023},
}

@article{rlhf,
  title={Deep reinforcement learning from human preferences},
  author={Christiano, Paul F and Leike, Jan and Brown, Tom and Martic, Miljan and Legg, Shane and Amodei, Dario},
  journal={Advances in neural information processing systems},
  volume={30},
  year={2017}
}

@article{bradley1952rank,
  title={Rank analysis of incomplete block designs: I. The method of paired comparisons},
  author={Bradley, Ralph Allan and Terry, Milton E},
  journal={Biometrika},
  volume={39},
  number={3/4},
  pages={324--345},
  year={1952},
  publisher={JSTOR}
}

@inproceedings{balseiro2021landscape,
  title={The landscape of auto-bidding auctions: Value versus utility maximization},
  author={Balseiro, Santiago R and Deng, Yuan and Mao, Jieming and Mirrokni, Vahab S and Zuo, Song},
  booktitle={Proceedings of the 22nd ACM Conference on Economics and Computation},
  pages={132--133},
  year={2021}
}

@article{duan2025adaptable,
  title={An Adaptable Budget Planner for Enhancing Budget-Constrained Auto-Bidding in Online Advertising},
  author={Duan, Zhijian and Huo, Yusen and Wang, Tianyu and Zhang, Zhilin and Li, Yeshu and Yu, Chuan and Xu, Jian and Zheng, Bo and Deng, Xiaotie},
  journal={arXiv preprint arXiv:2502.05187},
  year={2025}
}

@article{wang2023hibid,
  title={HiBid: A cross-channel constrained bidding system with budget allocation by hierarchical offline deep reinforcement learning},
  author={Wang, Hao and Tang, Bo and Liu, Chi Harold and Mao, Shangqin and Zhou, Jiahong and Dai, Zipeng and Sun, Yaqi and Xie, Qianlong and Wang, Xingxing and Wang, Dong},
  journal={IEEE Transactions on Computers},
  volume={73},
  number={3},
  pages={815--828},
  year={2023},
  publisher={IEEE}
}

@inproceedings{deng2023multi,
  title={Multi-channel autobidding with budget and roi constraints},
  author={Deng, Yuan and Golrezaei, Negin and Jaillet, Patrick and Liang, Jason Cheuk Nam and Mirrokni, Vahab},
  booktitle={International Conference on Machine Learning},
  pages={7617--7644},
  year={2023},
  organization={PMLR}
}

@article{kang2023efficient,
  title={Efficient diffusion policies for offline reinforcement learning},
  author={Kang, Bingyi and Ma, Xiao and Du, Chao and Pang, Tianyu and Yan, Shuicheng},
  journal={Advances in Neural Information Processing Systems},
  volume={36},
  pages={67195--67212},
  year={2023}
}

@article{zhu2023diffusion,
  title={Diffusion models for reinforcement learning: A survey},
  author={Zhu, Zhengbang and Zhao, Hanye and He, Haoran and Zhong, Yichao and Zhang, Shenyu and Guo, Haoquan and Chen, Tingting and Zhang, Weinan},
  journal={arXiv preprint arXiv:2311.01223},
  year={2023}
}

@article{goodfellow2020generative,
  title={Generative adversarial networks},
  author={Goodfellow, Ian and Pouget-Abadie, Jean and Mirza, Mehdi and Xu, Bing and Warde-Farley, David and Ozair, Sherjil and Courville, Aaron and Bengio, Yoshua},
  journal={Communications of the ACM},
  volume={63},
  number={11},
  pages={139--144},
  year={2020},
  publisher={ACM New York, NY, USA}
}

@inproceedings{pan2023better,
  title={Better training of gflownets with local credit and incomplete trajectories},
  author={Pan, Ling and Malkin, Nikolay and Zhang, Dinghuai and Bengio, Yoshua},
  booktitle={International Conference on Machine Learning},
  pages={26878--26890},
  year={2023},
  organization={PMLR}
}

@article{vaswani2017attention,
  title={Attention is all you need},
  author={Vaswani, Ashish and Shazeer, Noam and Parmar, Niki and Uszkoreit, Jakob and Jones, Llion and Gomez, Aidan N and Kaiser, {\L}ukasz and Polosukhin, Illia},
  journal={Advances in neural information processing systems},
  volume={30},
  year={2017}
}

@article{li2025generativemodelsdecisionmaking,
      title={Generative Models in Decision Making: A Survey}, 
      author={Yinchuan Li and Xinyu Shao and Jianping Zhang and Haozhi Wang and Leo Maxime Brunswic and Kaiwen Zhou and Jiqian Dong and Kaiyang Guo and Xiu Li and Zhitang Chen and Jun Wang and Jianye Hao},
      year={2025},
      journal={arXiv preprint arXiv:2502.17100}
}

@article{kingma2022,
      title={Auto-Encoding Variational Bayes}, 
      author={Diederik P Kingma and Max Welling},
      year={2022},
      journal={arXiv preprint arXiv:1312.6114}
}

@article{guan2021multiagent,
      title={Multi-Agent Cooperative Bidding Games for Multi-Objective Optimization in e-Commercial Sponsored Search}, 
      author={Ziyu Guan and Hongchang Wu and Qingyu Cao and Hao Liu and Wei Zhao and Sheng Li and Cai Xu and Guang Qiu and Jian Xu and Bo Zheng},
      year={2021},
      journal={arXiv preprint arXiv:2106.04075}
}

@inproceedings{lei2017alternating,
  title={Alternating pointwise-pairwise learning for personalized item ranking},
  author={Lei, Yu and Li, Wenjie and Lu, Ziyu and Zhao, Miao},
  booktitle={Proceedings of the 2017 ACM on Conference on Information and Knowledge Management},
  pages={2155--2158},
  year={2017}
}

@inproceedings{wang2022mp2,
  title={MP2: A momentum contrast approach for recommendation with pointwise and pairwise learning},
  author={Wang, Menghan and Guo, Yuchen and Zhao, Zhenqi and Hu, Guangzheng and Shen, Yuming and Gong, Mingming and Torr, Philip},
  booktitle={Proceedings of the 45th International ACM SIGIR Conference on Research and Development in Information Retrieval},
  pages={2105--2109},
  year={2022}
}

@article{mnih2015human,
  title={Human-level control through deep reinforcement learning},
  author={Mnih, Volodymyr and Kavukcuoglu, Koray and Silver, David and Rusu, Andrei A and Veness, Joel and Bellemare, Marc G and Graves, Alex and Riedmiller, Martin and Fidjeland, Andreas K and Ostrovski, Georg and others},
  journal={nature},
  volume={518},
  number={7540},
  pages={529--533},
  year={2015},
  publisher={Nature Publishing Group}
}

@article{ho2020denoising,
  title={Denoising diffusion probabilistic models},
  author={Ho, Jonathan and Jain, Ajay and Abbeel, Pieter},
  journal={Advances in neural information processing systems},
  volume={33},
  pages={6840--6851},
  year={2020}
}

@inproceedings{Alex2021improved,
  title={Improved Denoising Diffusion Probabilistic Models}, 
  author={Alex Nichol and Prafulla Dhariwal},
  year={2021},
  booktitle={International conference on machine learning},
  pages={8162--8171},
  organization={PMLR}
}

@inproceedings{ho2021classifier,
  title={Classifier-Free Diffusion Guidance},
  author={Ho, Jonathan and Salimans, Tim},
  booktitle={NeurIPS 2021 Workshop on Deep Generative Models and Downstream Applications},
  year={2021}
}

@inproceedings{ronneberger2015u,
  title={U-net: Convolutional networks for biomedical image segmentation},
  author={Ronneberger, Olaf and Fischer, Philipp and Brox, Thomas},
  booktitle={International Conference on Medical image computing and computer-assisted intervention},
  pages={234--241},
  year={2015},
  organization={Springer}
}

@article{wang2025reinforcement,
  title={Reinforcement Learning Optimization for Large-Scale Learning: An Efficient and User-Friendly Scaling Library},
  author={Wang, Weixun and Xiong, Shaopan and Chen, Gengru and Gao, Wei and Guo, Sheng and He, Yancheng and Huang, Ju and Liu, Jiaheng and Li, Zhendong and Li, Xiaoyang and others},
  journal={arXiv preprint arXiv:2506.06122},
  year={2025}
}
\bibliographystyle{iclr2026_conference}

\newpage
\tableofcontents

\newpage
\appendix
\section{Related Works}
\label{app:related_works}

\subsection{RL-based Auto-bidding Methods}
Auto-bidding plays a critical role in online advertising by automatically placing bids, allowing advertisers to participate efficiently in real-time auctions \citep{balseiro2021robust,deng2021towards,balseiro2021landscape}.
The auto-bidding problem can be modeled as a Markov Decision Process and addressed using reinforcement learning techniques. 
USCB \citep{uscb} proposes a unified solution to the constrained bidding problem, employing the RL method DDPG \citep{silver2014deterministic} to dynamically adjust parameters toward an optimal bidding strategy.
\citet{sorl} propose a sustainable online reinforcement learning framework that alternates between online exploration and offline training, thereby alleviating the sim2rel problem.
A few studies explore multi-agent RL for auto-bidding \citep{Jin_2018, guan2021multiagent, wen2022cooperative}, while several focus on budget allocation and bidding strategies in multi-channel scenarios using RL-based approaches \citep{wang2023hibid,deng2023multi,duan2025adaptable}.
Importantly, offline RL methods such as BCQ \citep{fujimoto2019off}, CQL \citep{kumar2020conservative}, IQL \citep{kostrikovoffline}, and MOPO \citep{yu2020mopo} have demonstrated significant potential in this domain.
These methods allow policy learning from pre-collected datasets without requiring online interaction.
Moreover, offline RL, such as Diffusion-QL \citep{wang2022diffusion}, employs a generative policy architecture to improve expressive capacity. 

However, RL-based methods often suffer from training instability arising from bootstrapping and alternating training between critics and actors.
Training instability typically deteriorates policy performance \citep{sutton1998reinforcement}. Moreover, training stability is even more critical in auto-bidding, given two domain-specific challenges: the absence of an accurate offline policy evaluation method and the high cost of online policy evaluation in a real-world advertising system \citep{sorl}.
Therefore, stable convergence to a well-performed policy is essential to ensure deployment reliability and system safety.

\subsection{Generative Auto-bidding Methods}
Generative models exhibit strong capabilities for capturing and reproducing underlying data distributions across a wide range of fields \citep{kingma2022,goodfellow2020generative,pan2023better,sohl2015deep,ho2020denoising,vaswani2017attention}.
They can be effectively incorporated into decision-making systems by generating complete trajectories that guide agents toward high-reward behaviors \citep{zhu2023diffusion,kang2023efficient,li2025generativemodelsdecisionmaking}.
In particular, Decision Transformer (DT) \citep{chen2021decision} reframes RL as a conditional sequence modeling problem and leverages transformer architectures to generate actions conditioned on desired returns, historical states, and actions. 
AIGB \citep{aigb} extends the generative perspective to the auto-bidding domain by formulating auto-bidding as a conditional generative modeling problem. 
DiffBid generates a state trajectory based on the desired return utilizing a conditional diffusion model, and then generates actions aligned with the optimized trajectory. 
These methods achieve superior performance in auto-bidding and offer distinct advantages over traditional RL methods. 
They do not rely on the bootstrapping mechanism commonly used in RL, thereby avoiding the instability caused by the deadly triad. 
Even so, these generative auto-bidding methods still face a performance bottleneck due to their neglect of fine-grained generation-quality evaluation and their inability to explore beyond static datasets.
In contrast, our method facilitates both reward evaluation and policy search through a learned trajectory evaluator.

\section{AIGB Method Details}
\label{app:aigb}
AIGB models the sequential decision-making problem via conditional diffusion, enabling effective trajectory generation for auto-bidding scenarios.
Specifically, AIGB utilizes the denoising diffusion probabilistic model (DDPM) \citep{ho2020denoising} for generation.
The forward and reverse processes are modeled as:
\begin{align}
    q(\tau_{k+1}|\tau_k), \quad p_\theta(\tau_k|\tau_{k+1}, y(\tau)),
\end{align}
respectively, where $q$ represents the forward noising process while $p_\theta$ the reverse denoising process. 

\textbf{Forward Process.}
In the forward process, the noise is gradually added to the latent variable by a Markov chain with pre-defined variance schedule $\beta_k$: 
\begin{align}
    q(\tau_k|\tau_{k-1})=\mathcal{N}(\tau_k;\sqrt{1-\beta_k}\tau_{k-1},\beta_kI)
\end{align}
where $k\in[K]$ refers to the diffusion step, $\tau_k\triangleq[s_1, s_2,\cdots,{s}_T]_k$ represents the latent variable in the $k$-th diffusion step, 
and $\tau_0$ is the original trajectory.
A notable property of the forward process is that $\tau_k$ at an arbitrary time-step $k$ can be sampled in closed form as:
\begin{equation}
    q(\tau_k|\tau_0)=\mathcal{N}(\tau_k;\sqrt{\bar\alpha_k}\tau_0,(1-\bar\alpha_k)I),
\end{equation}
where $\alpha_k=1-\beta_k$ and $\bar\alpha_k=\prod^k_{i=1}\alpha_k$.
When $k \rightarrow \infty$, $\tau_k$ approaches a standard Gaussian distribution.
In particular, AIGB employs a cosine noise schedule \citep{Alex2021improved} to control the schedule $\beta_k$.

\textbf{Reverse Process.}
In the reverse process, 
diffusion models aim to remove the added noise on $\tau_K$ and recursively recover $\tau_{0}$.
This process is governed by the conditional model $p_\theta(\tau_{k-1}|\tau_k,y(\tau))$, which is parameterized through a noise prediction model $\epsilon_\theta(\tau_k,y(\tau),k)$.
AIGB adopts a temporal U-Net \citep{ronneberger2015u} for the noise prediction model, a common choice in diffusion-based decision-making methods \citep{ajay2023is}.  

\subsection{Training Stage}
The training of the diffusion model is typically formulated as minimizing the mean squared error between the predicted noise $\epsilon_{\theta}$ and the true noise applied during the forward diffusion process. 
Specifically, during each iteration, we randomly sample a trajectory from the offline dataset $\mathcal{D}$ and pick a time step $t\in[T]$. We recursively add the Gaussian noise $\epsilon$ to the states in $\tau$ with time steps bigger than $t$ and predict the added noises with $\epsilon_\theta(\tau_k, y(\tau), k)$, where the states between $0$ and $t$ in $\tau_k$ are set to real history states $s_1,s_2,\cdots,s_t$. 
In addition to this standard objective, AIGB also incorporates a supervised loss that measures the discrepancy between the true actions and the actions predicted by an inverse dynamics model $\hat{f}_{\phi}(s_t, \hat{s}_{t+1})$.
Overall, the complete training objective of AIGB can be expressed as: 
\begin{equation}
    \mathcal{L}(\theta,\phi)=\mathbb{E}_{k,\tau\in\mathcal{D}}[||\epsilon-\epsilon_\theta(\tau_k,y(\tau),k)||^2]+\mathbb{E}_{(s_t,a_t,\hat{s}_{t+1})\in\mathcal{D}}[||a_t-\hat{f}_{\phi}(s_t, \hat{s}_{t+1})||^2].
\end{equation}
During training, the condition $y(\tau)$ is randomly dropped to enhance model robustness. 
This technique ensures that both the unconditional model $\epsilon_\theta(\tau_k,k)$ and the conditional model $\epsilon_\theta(\tau_k,y(\tau),k)$ are effectively trained together.

\subsection{Inference Stage}
Starting with Gaussian noise, trajectories are iteratively generated through a series of denoising steps.
Specifically, AIGB uses a classifier-free guidance strategy \citep{ho2021classifier} to guide the generation of bidding and extract high-likelihood trajectories in the dataset.
During  generation, AIGB combines conditional and unconditional  score estimates linearly: 
\begin{equation}
\hat{\epsilon}_k:=\epsilon_\theta(\tau_k,k)+\omega\left( \epsilon_\theta\left( \tau_k,y(\tau),k \right) - \epsilon_\theta \left ( \tau_k, k \right) \right),
\label{eq:perturbed_noise}
\end{equation}
where $\omega$ is the guidance scale that controls the influence of the condition $y(\tau)$. 
This formulation effectively steers the trajectory generation toward regions of the data distribution that are most consistent with the given condition.
The predicted state at each step is sampled from 
$p_\theta(\tau_{k-1}|\tau_k,y(\tau))$:
\begin{equation}
    \tau_{k-1}\sim\mathcal{N}\left (\tau_{k-1}|{\mu}_{\theta}\left(\tau_k,y(\tau),k\right),{\Sigma}_{\theta}\left(\tau_k,k\right ) \right ),
\end{equation}
with mean and variance defined as  ${\mu}_\theta(\tau_k,y(\tau),k)=\frac{1}{\sqrt{\alpha_k}}(\tau_k-\frac{\beta_k}{\sqrt{1-\overline{\alpha}_k}}\hat{\epsilon}_k )$ and $\Sigma_\theta(\cdot)=\beta_k$. 
Note that the initial noisy trajectory  $\tau'_K\sim\mathcal{N}(0,I)$ is assigned with history states $s_{1:t}$ for the first $t$ states to ensure history consistency.  This is consistent with the training process. 
By recursively applying the reverse diffusion process using:
\begin{equation}
    \tau'_{k-1}={\mu}_\theta(\tau'_k,y(\tau), k)+\sqrt{\beta_k}{z},
    \label{eq:sample}
\end{equation}
where $z\sim\mathcal{N}(0,I)$, we obtain the final denoised trajectory $\tau'_0$, from which the next state $\hat{s}_{t+1}$ is derived. 
Then the action is generated through an inverse dynamics $\hat{a}_t=\hat{f}_{\phi}(s_t, \hat{s}_{t+1})$.

\section{Theoretical Proofs}
\setcounter{theorem}{0}
\subsection{Proof of Theorem \ref{thm:lipschitz_continuous}.}
\label{app:proof_thm1}

\begin{figure}
  \centering
\includegraphics[width=0.9\linewidth]{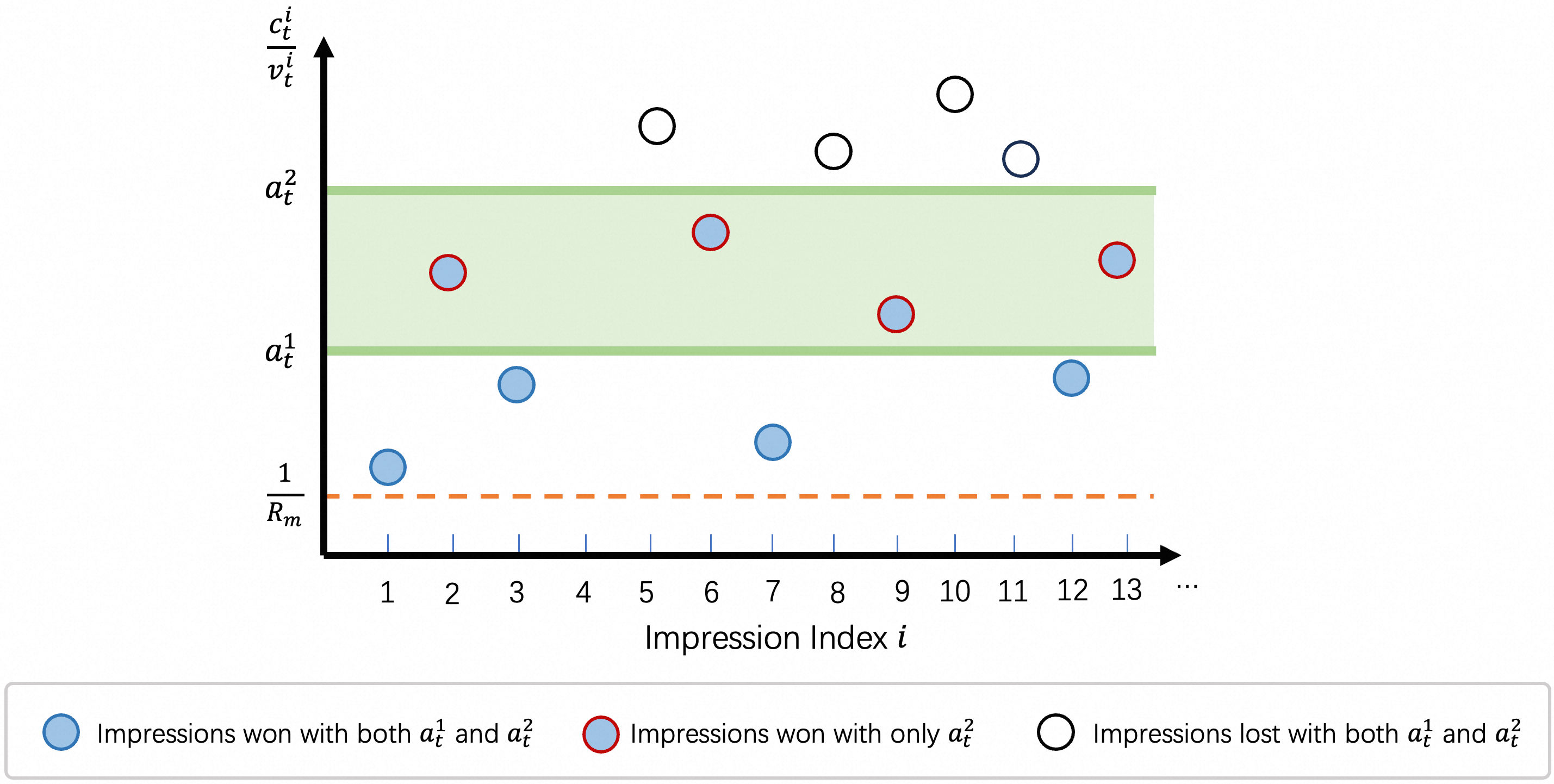}
  \caption{The impression opportunities within time step $t$ and $t+1$, where $p_t^i/v_t^i$ is the $1/\text{ROI}$ of impression $i$. Without loss of generality, consider two actions $a_{1,t}$ and $a_{2,t}$, and let $a_{2,t}\ge a_{1,t}$. The impressions within the shadow area are the impressions won by action $a_{2,t}$ but lost by action $a_{1,t}$.  
  }
  \label{fig:impressions}
\end{figure}

\begin{theorem}[Lipschitz Continuous of $y(\tau)$.]
The trajectory quality $y(\tau)$ is $\sqrt{T}R_m$-Lipschitz continuous with respect to the Frobenius norm.
\end{theorem}
\begin{proof} 
Recall from Section \ref{sec:prob_state} that the cost $c_t$ and reward $r_t$ under action $a_t$ between time step $t$ and $t+1$ can be written as:
\begin{align}
    c_t=\sum_{i}\mathbbm{1}\bigg\{a_t\ge \frac{p_t^i}{v_t^i}\bigg\}p_t^i\quad\text{and}\quad r_t=\sum_{i}\mathbbm{1}\bigg\{a_t\ge \frac{p_t^i}{v_t^i}\bigg\}v_t^i,
\end{align}
where $p_t^i$ and $v_t^i$ denote the market price and the value of the $i$-th impression between time step $t$ and $t+1$. Accordingly, the cost ratio $\bar{c}_t$ and the normalized reward $\bar{r}_t$ can be written as:
\begin{align}
     \bar{c}_t=\frac{1}{B}\sum_{i}\mathbbm{1}\bigg\{a_t\ge \frac{p_t^i}{v_t^i}\bigg\}p_t^i\quad\text{and}\quad \bar{r}_t=\frac{1}{B}\sum_{i}\mathbbm{1}\bigg\{a_t\ge \frac{p_t^i}{v_t^i}\bigg\}v_t^i,
\end{align}
Consider two different trajectories $\tau_1$ and $\tau_2$ with actions, cost ratios and normalized rewards sequences $\{a_{1,t}, \bar{c}_{1,t}, \bar{r}_{1,t}\}_{t=1}^T$ and $\{a_{2,t}, \bar{c}_{2,t}, \bar{r}_{2,t}\}_{t=1}^T$, respectively. 
The trajectory quality gap between $\tau_1$ and $\tau_2$ holds that:
\begin{align}
\label{equ:le_r}
    |y(\tau_1)-y(\tau_2)|=\big|\sum_t \bar{r}_{1,t}-\sum_t\bar{r}_{2,t}\big|\le \sum_{t}|\bar{r}_{1,t}-\bar{r}_{2,t}|.
\end{align}
Consider the reward gap between time step $t$ and $t+1$, as shown in Fig.\ref{fig:impressions}.
Without loss of generality, let $a_{2,t}\ge a_{1,t}$. We have:
\begin{align}
\label{equ:r_in}
    |\bar{r}_{1,t}-\bar{r}_{2,t}|&=\frac{1}{B}\sum_i\bigg[\mathbbm{1}\bigg\{a_{2,t}\ge \frac{p_t^i}{v_t^i}\bigg\} - \mathbbm{1}\bigg\{a_{1,t}\ge \frac{p_t^i}{v_t^i}\bigg\} \bigg] v_t^i\notag\\
    &=\frac{1}{B}\sum_i \mathbbm{1}\bigg\{a_{2,t}\ge \frac{p_t^i}{v_t^i} \ge a_{1,t}\bigg\}v_t^i\notag\\
    &=\frac{1}{B}\sum_i \mathbbm{1}\bigg\{a_{2,t}\ge \frac{p_t^i}{v_t^i} \ge a_{1,t}\bigg\}\frac{v_t^i}{p_t^i}p_t^i\notag\\
    &\le \frac{R_m}{B} \sum_i \mathbbm{1}\bigg\{a_{2,t}\ge \frac{p_t^i}{v_t^i} \ge a_{1,t}\bigg\}p_t^i.
\end{align}
Note that the cost ratio gap between time step $t$ and $t+1$ can be written as:
\begin{align}
\label{equ:c_in}
    |\bar{c}_{1,t}-\bar{c}_{2,t}|=\frac{1}{B}\sum_i\bigg[\mathbbm{1}\bigg\{a_{2,t}\ge \frac{p_t^i}{v_t^i}\bigg\} - \mathbbm{1}\bigg\{a_{1,t}\ge \frac{p_t^i}{v_t^i}\bigg\} \bigg] p_t^i=\frac{1}{B}\sum_i \mathbbm{1}\bigg\{a_{2,t}\ge \frac{p_t^i}{v_t^i} \ge a_{1,t}\bigg\}p_t^i.
\end{align}
Therefore, combining Eq. \ref{equ:r_in} and Eq. \ref{equ:c_in}, we have:
\begin{align}
\label{equ:r_c}
    |\bar{r}_{1,t}-\bar{r}_{2,t}| \le R_m |\bar{c}_{1,t}-\bar{c}_{2,t}|.
\end{align}
We examine the Frobenius norm of the gap between $\tau_1$ and $\tau_2$:
\begin{align}
\label{equ:cs}
    \|\tau_1-\tau_2\|_F&=\bigg\|\begin{bmatrix}1&\bar{c}_{1,0}& x\\2&\bar{c}_{1,1}& x\\ 
    \vdots&\vdots&\vdots\\
    T&c_{1,T-1}& x\end{bmatrix}-\begin{bmatrix}1&\bar{c}_{2,0}& x\\2&\bar{c}_{2,1}& x\\ 
    \vdots&\vdots&\vdots\\
    T&\bar{c}_{2,T-1}& x\end{bmatrix}\bigg\|_F\notag\\
    &=\sqrt{\sum_{t}(\bar{c}_{1,t}-\bar{c}_{2,t})^2}\notag\\
    &\ge \frac{1}{\sqrt{T}}\sum_t |\bar{c}_{1,t}-\bar{c}_{2,t}|\qquad \qquad\text{(Cauchy-Schwarz Inequality)}
\end{align}
Combining Eq. \ref{equ:le_r}, Eq. \ref{equ:r_c} and Eq. \ref{equ:cs}, we can obtain that:
\begin{align}
    |y(\tau_1)-y(\tau_2)|&\le \sum_t|\bar{r}_{1,t}-\bar{r}_{2,t}|\notag\\&\le R_m\sum_t|\bar{c}_{1,t}-\bar{c}_{2,t}|\notag\\
    &\le \sqrt{T}R_m\frac{1}{\sqrt{T}}\sum_t|\bar{c}_{1,t}-\bar{c}_{2,t}|\notag\\
    &\le \sqrt{T}R_m \left\|\tau_1-\tau_2\right\|_F.
\end{align}
This concludes the proof.
\end{proof}

\subsection{Proof of Theorem \ref{thm:performance_gap_bound}}
\label{app:proof_thm2}

Here, we list two lemmas used in the proof of Theorem \ref{thm:performance_gap_bound}.

\begin{lemma}[Additivity of the Lipschitz]
\label{lemma:lipschitz}
    Let $f_1(x)$ and $f_2(x)$ be two Lipschitz continuous functions with Lipschitz constants $L_1>0$ and $L_2>0$, respectively. Then $|f_1(x)+f_2(x)|$ is also a Lipschitz continuous function, with Lipschitz constant at most $L_1+L_2$.    
\end{lemma}
\begin{proof}
Recall the Reverse Triangle Inequality states that $\forall a,b$, we have $||a|-|b||\le |a-b|$.
Then, $\forall x,y$, we have:
\begin{align}
    ||f_1(x)+f_2(x)|-|f_1(y)+f_2(y)||&\le |f_1(x)+f_2(x)-f_1(y)-f_2(y)|\notag\\
    &\le |f_1(x)-f_1(y)|+|f_2(x)-f_2(y)|\notag\\
    &\le (L_1+L_2)|x-y|.
\end{align}
This concludes the proof.
\end{proof}

\begin{lemma}[Kantorovich-Rubinstein Duality Theorem \citep{villani2021topics}]
\label{lemma:KR}
Let $(X, d)$ be a metric space, and let $p$ and $q$ be two probability distributions on $X$. 
Let $f:X\rightarrow \mathbb{R}$ be an $L$-Lipschitz function, and $W_1(p,q)$ denotes the 1-Wasserstein distance between $p$ and $q$. Then we have:
\begin{align}
    |\mathbb{E}_{x\sim p} f(x)-\mathbb{E}_{x\sim q} f(x)|\le L\cdot W_1(p,q).
\end{align}
\end{lemma}
We next give the proof of Theorem \ref{thm:performance_gap_bound}.

\begin{theorem}[Evaluator Bias in Planning Performance Bound]
Let the upper bound of the evaluator's bias on its training dataset $\mathcal{D}$ be $\delta_D>0$. 
The gap between the planner's score $L(\theta)$ and its true performance $J(\theta)$ can be bounded by:
\begin{align}
   |J(\theta)-L(\theta)|\le \delta_D +(1+k)\sqrt{T}R_m\mathbb{E}_{y\sim p_D(y)}\bigg[\underbrace{W_1(p_\theta(\tau|y^*), p_\theta(\tau|y))}_{\text{ Lipschitz sensitivity to $y$}} + \underbrace{W_1(p_\theta(\tau|y), p_D(\tau|y))}_{\text{imitation error on $\mathcal{D}$}}\bigg],\notag
\end{align}
where $W_1$ denotes the 1-Wasserstein distance.
\end{theorem}
\begin{proof}
The evaluator bias in the planner's performance can be written as:
\begin{align}
     |J(\theta)-L(\theta)| = |\mathbb{E}_{\tau\sim p_\theta(\tau | y^*)}[y(\tau)-\hat{y}_\phi(\tau)]|\le \mathbb{E}_{\tau\sim p_\theta(\tau | y^*)}\underbrace{|y(\tau)-\hat{y}_\phi(\tau)|}_{\triangleq f(\tau)}
\end{align}
Let $f(\tau)\triangleq |y(\tau)-\hat{y}_\phi(\tau)|$ be the evaluator bias in trajectory $\tau$. 
From Theorem \ref{thm:lipschitz_continuous} and Lemma \ref{lemma:lipschitz}, we know that $f(\tau)$ is a $(1+k)\sqrt{T}R_m$-Lipschitz continuous function.
Then, we have:
\begin{align}
\label{equ:bias}
    |J(\theta)-L(\theta)|&\le\mathbb{E}_{\tau\sim p_\theta(\tau| y^*)} f(\tau)\notag\\
    &=\mathbb{E}_{y\sim p_D(y)}\bigg[\mathbb{E}_{\tau\sim p_\theta(\tau| y^*)} f(\tau)-\mathbb{E}_{\tau\sim p_D(\tau|y)}f(\tau)+\mathbb{E}_{\tau\sim p_D(\tau|y)}f(\tau)\bigg]\notag\\
    &=\underbrace{\mathbb{E}_{y\sim p_D(y)}\mathbb{E}_{\tau\sim p_D(\tau|y)}f(\tau)}_{\le \delta_D} + \mathbb{E}_{y\sim p_D(y)}\bigg[\mathbb{E}_{\tau\sim p_\theta(\tau| y^*)} f(\tau)-\mathbb{E}_{\tau\sim p_D(\tau|y)}f(\tau)\bigg]\notag\\
    &=\delta_D + \mathbb{E}_{y\sim p_D(y)}\underbrace{\bigg[\mathbb{E}_{\tau\sim p_\theta(\tau| y^*)} f(\tau) - \mathbb{E}_{\tau\sim p_\theta(\tau| y)} f(\tau) \bigg ]}_{\le (1+k)\sqrt{T}R_mW_1(p_\theta(\tau|y^*), p_\theta(\tau|y)), \text{ (Lemma \ref{lemma:KR})}}\notag\\
    &\qquad\quad+ \mathbb{E}_{y\sim p_D(\tau)}\underbrace{\bigg[ \mathbb{E}_{\tau\sim p_\theta(\tau| y)} f(\tau) -\mathbb{E}_{\tau\sim p_D(\tau|y)}f(\tau)\bigg]}_{\le (1+k)\sqrt{T}R_mW_1(p_\theta(\tau|y), p_D(\tau|y)), \text{ (Lemma \ref{lemma:KR})}}\notag\\
    &\le \delta_D +(1+k)\sqrt{T}R_m\mathbb{E}_{y\sim p_D(y)}[W_1(p_\theta(\tau|y^*), p_\theta(\tau|y))] \notag\\
    &\qquad\quad+ (1+k)\sqrt{T}R_m\mathbb{E}_{y\sim p_D(y)}[W_1(p_\theta(\tau|y), p_D(\tau|y))].
\end{align}
Therefore, we have:
\begin{align}
   |J(\theta)-L(\theta)|\le \delta_D +(1+k)\sqrt{T}R_m\mathbb{E}_{y\sim p_D(y)}\bigg[W_1(p_\theta(\tau|y^*), p_\theta(\tau|y)) + W_1(p_\theta(\tau|y), p_D(\tau|y))\bigg].
\end{align}
This concludes the proof.
\end{proof}

\subsection{Proof of Eq. \ref{equ:lipschiz_planner_bound}}
\label{app:proof_lipschitz_planner}
We give the proof of Eq. \ref{equ:lipschiz_planner_bound} as follows. 
Denote $\text{Lip}_{W_1}(p_\theta(\tau|y))$ as the planner's Lipschitz constant with respect to $y$ regarding the Wasserstein distance $W_1$, we have:
\begin{align}
     \mathbb{E}_{y\sim p_D(y)}[W_1(p_\theta(\tau|y^*), p_\theta(\tau|y))]&\le \text{Lip}_{W_1}(p_\theta(\tau|y))\mathbb{E}_{y\sim p_D(y)}[ ((1+\epsilon)y_m-y)]\notag \\
    &=\text{Lip}_{W_1}(p_\theta(\tau|y))\int_{0}^{y_m} p_D(y)[(1+\epsilon)y_m - y] \mathrm{d}y\notag\\
    &\le \text{Lip}_{W_1}(p_\theta(\tau|y)) \int_0^{y_m}p_D(y)[(1+\epsilon)y_m ] \mathrm{d}y\notag\\
    &= (1+\epsilon)y_m\text{Lip}_{W_1}(p_\theta(\tau|y)),
\end{align}
where we leverage the non-negativity property of the condition $y\ge 0, \forall y\in\mathcal{D}$.
This completes the proof.

\subsection{Proof of Eq. \ref{equ:kl_bound}}
\label{app:proof_kl_bound}

\begin{lemma}[Pinsker's Inequality \citep{Tsybakov2008introduction}]
\label{lemma:pinkser}
Let $P$ and $Q$ be two probability measures defined on the same measurable space, and assume that $P$ is absolutely continuous with respect to $Q$, i.e., $P\ll Q$. Then the total variation distance between $P$ and $Q$ is bounded above by the KL divergence from $P$ to $Q$ as follows:
\begin{align}
    \|P-Q\|_\text{TV}\le\sqrt{\frac{1}{2}D_{KL}(P\|Q)}.
\end{align}
\end{lemma}

\begin{lemma}[Wasserstein–Total Variation Inequality on Bounded Metric Spaces \citep{villani2008optimal}]
\label{lemma:W_1_TV}
Let $(\mathcal{Z}, d)$ be a metric space with diameter $\text{diam}(\mathcal{Z})\triangleq\sup_{z_1,z_2\in\mathcal{Z}}d(z_1,z_2)$. Let $P$ and $Q$ be two probability measures on $\mathcal{Z}$. Then the 1-Wasserstein distance between $P$ and $Q$ satisfies:
\begin{align}
    W_1(P,Q)\le \text{diam}(\mathcal{Z})\|P-Q\|_\text{TV}.
\end{align}
    
\end{lemma}

We give the proof of Eq. \ref{equ:kl_bound} as follows.
Equipped with the above two lemmas, we have:
\begin{align}
\label{equ:W_1_TV_KL}
    W_1(p_\theta(\tau|y), p_D(\tau|y))&\le \text{diam}(\mathcal{T})\|p_\theta(\tau|y)-p_D(\tau|y)\|_\text{TV}\notag\\
    &\le \text{diam}(\mathcal{T})\sqrt{\frac{1}{2}D_{KL}(p_D(\tau|y)\|p_\theta(\tau|y))},
\end{align}
where $\mathcal{T}$ is the trajectory space. Note that due to the budget constraint $\sum_t c_t\le B$ \footnote{As explained in Footnote 1, the budget constraint is guaranteed to be satisfied in real-world advertising systems thanks to an automatic suspension mechanism that halts bidding once the budget is exhausted.}, we have the sum of the cost ratio satisfies $\sum_t \bar{c}_t\le 1$.
The trajectory space can be expressed as:
\begin{align}
    \mathcal{T}=\bigg\{\bigg[[1, \bar{c}_{0},x], [2,\bar{c}_{1}, x], \cdots,[T, \bar{c}_{T-1}, x]\bigg]\bigg| \bar{c}_t\ge 0, \forall t, \text{and} \sum_t \bar{c}_t\le 1\bigg\}
\end{align}
We next prove that the diameter of the trajectory space, $\text{diam}(\mathcal{T})$, can be bounded by a constant. Specifically, the diameter only depends on the largest possible distance between the cost ratio sequences in two trajectories since:
\begin{align}
    \text{diam}(\mathcal{T})&=\sup_{\tau_1,\tau_2\in \mathcal{T}} \|\tau_1-\tau_2\|_F\notag\\
    &=\sup_{\tau_1,\tau_2\in \mathcal{T}} \bigg\|\begin{bmatrix}1&\bar{c}_{1,0}& x\\2&\bar{c}_{1,1}& x\\ 
    \vdots&\vdots&\vdots\\
    T&c_{1,T-1}& x\end{bmatrix}-\begin{bmatrix}1&\bar{c}_{2,0}& x\\2&\bar{c}_{2,1}& x\\ 
    \vdots&\vdots&\vdots\\
    T&\bar{c}_{2,T-1}& x\end{bmatrix}\bigg\|_F\notag\\
    &=\sup_{\tau_1,\tau_2\in \mathcal{T}}\sqrt{\sum_{t}(\bar{c}_{1,t}-\bar{c}_{2,t})^2}.
\end{align}
For convenience, we let $\mathbf{c}_i\triangleq [\bar{c}_{i,0}, \bar{c}_{i,2}, \cdots, \bar{c}_{i,T-1}], i\in\{1,2\}$. Then, the key part in the above result, $\sum_{t}(\bar{c}_{1,t}-\bar{c}_{2,t})^2$, can be written as:
\begin{align}
\label{equ:inner_product}
    \sum_{t}(\bar{c}_{1,t}-\bar{c}_{2,t})^2&=\sum_t(\bar{c}_{1,t}^2-2\bar{c}_{1,t}\bar{c}_{2,t}+\bar{c}_{2,t}^2)\notag\\
    &=\|\mathbf{c}_1\|^2_2+\|\mathbf{c}_2\|^2_2-2\langle\mathbf{c}_1,\mathbf{c}_2\rangle\notag\\
    &\le \|\mathbf{c}_1\|^2_2+\|\mathbf{c}_2\|^2_2,
\end{align}
where $\langle\mathbf{c}_1,\mathbf{c}_2\rangle\ge 0$.
As $\bar{c}_{i,t}\ge0$ and $\sum_{t}\bar{c}_{i,t}\le 1$, we have $0\le \bar{c}_{i,t}\le 1$. Therefore, it holds that:
\begin{align}
\label{equ:zero_to_one}
\|\mathbf{c}_i\|^2_2=\sum_{t}\bar{c}_{i,t}^2\le \sum_t \bar{c}_{i,t}\le 1.
\end{align}
Combining Eq. \ref{equ:inner_product} and Eq. \ref{equ:zero_to_one}, we have:
\begin{align}
    \sqrt{\sum_{t}(\bar{c}_{1,t}-\bar{c}_{2,t})^2}\le \sqrt{\|\mathbf{c}_1\|^2_2+\|\mathbf{c}_2\|^2_2}\le \sqrt{2}.
\end{align}
Therefore, we have $\text{diam}(\mathcal{T})=\sqrt{2}$. According to Eq. \ref{equ:W_1_TV_KL}, we have:
\begin{align}
\label{equ:W_1_KL}
     W_1(p_\theta(\tau|y), p_D(\tau|y))\le \sqrt{D_{KL}(p_D(\tau|y)\|p_\theta(\tau|y))}.
\end{align}
Recall that we impose the KL-constraint as: 
\begin{align}
\label{equ:KL_constraint}
    \mathbb{E}_{y\sim p_D(y)}[D_\text{KL}(p_D(\tau|y) \| p_\theta(\tau|y))]\le \delta_K,
\end{align}
Taking the expectation over $y\sim p_D(y)$ on both sides of Eq. \ref{equ:W_1_KL}, we have:
\begin{align}
\label{equ:W_2}
     \mathbb{E}_{y\sim p_D(y)}[W_1(p_\theta(\tau|y), p_D(\tau|y))]&\le \mathbb{E}_{y\sim p_D(y)}\bigg[\sqrt{ D_{KL}(p_D(\tau|y) \| p_\theta(\tau|y))}\bigg]\notag\\
     &\le \sqrt{\mathbb{E}_{y\sim p_D(y)}[ D_{KL}(p_D(\tau|y) \| p_\theta(\tau|y))]} \quad\text{(Jensen Inequality)}\notag\\
     &\le \sqrt{\delta_K}\qquad\qquad\qquad\qquad\qquad\qquad\qquad\text{(KL  constraint Eq.\ref{equ:KL_constraint})}.
\end{align}
This completes the proof.

\subsection{Proof of Theorem \ref{thm:sub_optimality}}
\label{app:proof_suboptimality}
\begin{theorem}[Sub-optimality Gap Bound] 
Let $\delta_M\triangleq \mathbb{E}_{y\sim p_D(y)}[D_{KL}(p_D(\tau|y)\|p_{\theta^*}(\tau|y))]$ be the expected distance between the optimal trajectory distribution and the trajectory distribution of the offline dataset $\mathcal{D}$.
The true performance gap between the optimal parameter $\theta^*$ and the solution $\hat{\theta}$ to Eq. \ref{equ:max_R_constrained}  is bounded by:
\begin{align}
    J(\theta^*) - J(\hat{\theta})\le 2\delta_D+(1+2k)\sqrt{T}R_m\bigg[\sqrt{\delta_M}+\sqrt{\delta_K}+(1+\epsilon)y_mL_p\bigg].
\end{align}
\end{theorem}
\begin{proof}
    The sub-optimality gap can be expressed as follows:
    \begin{align}
    \label{equ:decompose}
        J(\theta^*)-J(\hat{\theta}) &= \big(J(\theta^*)-L(\theta^*)\big) + \big(L(\theta^*) - L(\hat{\theta})\big)+\big( L(\hat{\theta})-J(\hat{\theta})\big)\notag\\
        &\le \underbrace{|J(\theta^*)-L(\theta^*)| }_{\text{evaluator bias in }p_{\theta^*}}
        + \underbrace{|L(\theta^*)-L(\hat{\theta})|}_{\text{score gap}}
        + \underbrace{|L(\hat{\theta})-J(\hat{\theta})|}_{\text{evaluator bias in } p_{\hat{\theta}}}.
    \end{align}
We examine the above three terms accordingly.

\textbf{(1) Evaluator Bias in $p_{\theta^*}$.} Denote the evaluator bias on trajectory $\tau$ as $f(\tau)\triangleq |y(\tau)-\hat{y}_\phi(\tau)|$. Following the derivation process in Eq. \ref{equ:bias}, we have:
\begin{align}
     |J(\theta^*)-L(\theta^*)|
    &\le \mathbb{E}_{y\sim p_D(y)}\mathbb{E}_{\tau\sim p_D(\tau|y)}f(\tau) + \mathbb{E}_{y\sim p_D(y)}\bigg[\mathbb{E}_{\tau\sim p_{\theta^*}(\tau| y^*)} f(\tau)-\mathbb{E}_{\tau\sim p_D(\tau|y)}f(\tau)\bigg]\notag\\
    &\le \delta_D +(1+k)\sqrt{T}R_m\mathbb{E}_{y\sim p_D(y)}[W_1(p_{\theta^*}(\tau|y^*), p_D(\tau|y))],
\end{align}
where $W_1(p_{\theta^*}(\tau|y^*), p_D(\tau|y))$ denotes the probability distribution distance between the optimal planner and the offline dataset.
Based on the derivation in Appendix \ref{app:proof_kl_bound}, we have:
\begin{align}
\label{equ:W_1_theta*}
    \mathbb{E}_{y\sim p_D(y)}[W_1(p_{\theta^*}(\tau|y^*), p_D(\tau|y))]\le \sqrt{\mathbb{E}_{y\sim p_D(y)}[D_{KL}(p_D(\tau|y) \| p_{\theta^*}(\tau|y^*) )]}
\end{align}
Let $\delta_M\triangleq \mathbb{E}_{y\sim p_D(y)}[D_{KL}(p_D(\tau|y) \| p_{\theta^*}(\tau|y^*) )]$ be the distance between the optimal trajectory distribution and the offline dataset trajectory distribution. We have:
\begin{align}
\label{equ:gap_1}
    |J(\theta^*)-L(\theta^*)| \le \delta_D +(1+k)\sqrt{T}R_m\sqrt{\delta_M}.
\end{align}

\textbf{(2) Score Gap.} Recall that the trained evaluator $\hat{y}_\phi(\tau)$ is a $k\sqrt{T}R_m$-Lipschitz continuous function with the Lipschitz constraint design. With Lemma \ref{lemma:KR}, we have:
\begin{align}
\label{equ:gap_2}
    |L(\theta^*)-L(\hat{\theta})|&=|\mathbb{E}_{\tau\sim p_{\theta^*}(\tau|y^*)}\hat{y}_\phi(\tau)-\mathbb{E}_{\tau\sim p_{\hat{\theta}}(\tau|y^*)}\hat{y}_\phi(\tau)|\notag\\
    &\le  k\sqrt{T}R_m W_1(p_{\theta^*}(\tau|y^*), p_{\hat{\theta}}(\tau|y^*))\notag\\
    &\le k\sqrt{T}R_m \mathbb{E}_{y\sim p_D(y)}\bigg[W_1(p_{\theta^*}(\tau|y^*), p_D(\tau|y)) +W_1(p_D(\tau|y), p_{\hat{\theta}}(\tau | y^*))\bigg]\notag\\
    & \le k\sqrt{T}R_m\bigg[\sqrt{\delta_M}+(1+\epsilon)y_mL_p + \sqrt{\delta_K}\bigg],
\end{align}
where we leverage the fact that $\hat{\theta}$ is a solution to Eq. \ref{equ:max_R_constrained} which satisfies the KL and Lipschitz constraint, and we leverage the results in Eq. \ref{equ:lipschiz_planner_bound}, Eq. \ref{equ:kl_bound} and Eq. \ref{equ:W_1_theta*}.

\textbf{(3) Evaluator Bias in $p_{\hat{\theta}}$. }
Since $\hat{\theta}$ satisfies the KL and Lipschitz constraint in Eq. \ref{equ:max_R_constrained}, we can use the results in Theorem \ref{thm:performance_gap_bound}, Eq. \ref{equ:lipschiz_planner_bound} and Eq. \ref{equ:kl_bound} to obtain:
\begin{align}
\label{equ:gap_3}
 |J(\hat{\theta})-L(\hat{\theta})|\le \delta_D +(1+k)\sqrt{T}R_m[(1+\epsilon)y_mL_p+\sqrt{\delta_K}],
\end{align}

Overall, combining the results in Eq. \ref{equ:gap_1}, Eq. \ref{equ:gap_2}, and Eq. \ref{equ:gap_3}, we have:
\begin{align}
    J(\theta^*) - J(\hat{\theta})\le 2\delta_D+(1+2k)\sqrt{T}R_m\bigg[\sqrt{\delta_M}+\sqrt{\delta_K}+(1+\epsilon)y_mL_p\bigg].
\end{align}
This concludes the proof.

\end{proof}

\subsection{Proof of Score Gradient}
\label{app:proof_score_gradient}
The probability of the trajectory generated by the Causal Transformer can be decomposed into:
\begin{align}
    p_\theta(s_{1:T}|y) = \Pi_t p_\theta(s_t|s_{1:t-1},y).
\end{align}
Then, we have:
\begin{align}
\label{equ:proof_pg}
    \nabla_\theta L(\theta)
    &=\nabla_\theta\int_{\tau}p_\theta(\tau|y^*)\hat{y}_\phi(\tau)\mathrm{d}\tau\notag\\
&=\nabla_\theta\int_{s_1,\cdots,s_{T}}p_\theta(s_{1:T},\tau|y^*)\hat{y}_\phi(\tau)\mathrm{d}s_1\cdots\mathrm{d}s_T\notag\\
    &=\int_{s_{1:T}}p_\theta(s_{1:T}|y^*)\frac{\nabla_\theta p_\theta(s_{1:T}|y^*)}{p_\theta(s_{1:T}|y^*)}\hat{y}_\phi(\tau)\mathrm{d}s_1\cdots\mathrm{d}s_T\notag\\
    &=\mathbb{E}_{s_{1:T}\sim p_\theta(s_{1:T}|y^*)}\bigg[\nabla_\theta\log p_\theta(s_{1:T}|y^*) \hat{y}_\phi(\tau)\bigg]\notag\\
    &=\mathbb{E}_{s_{1:T}\sim p_\theta(s_{1:T}|y^*)}\bigg[\nabla_\theta\log \Pi_{t} p_\theta(s_{t}|s_{1:t-1},y^*) \hat{y}_\phi(\tau)\bigg]\notag\\
    &=\mathbb{E}_{s_{1:T}\sim p_\theta(s_{1:T}|y^*)}\bigg[\sum_{t}\nabla_\theta\log p_\theta(s_{t}|s_{1:t-1},y^*) \hat{y}_\phi(\tau)\bigg].
\end{align}

\begin{figure}
    \centering
    \includegraphics[width=0.7\linewidth]{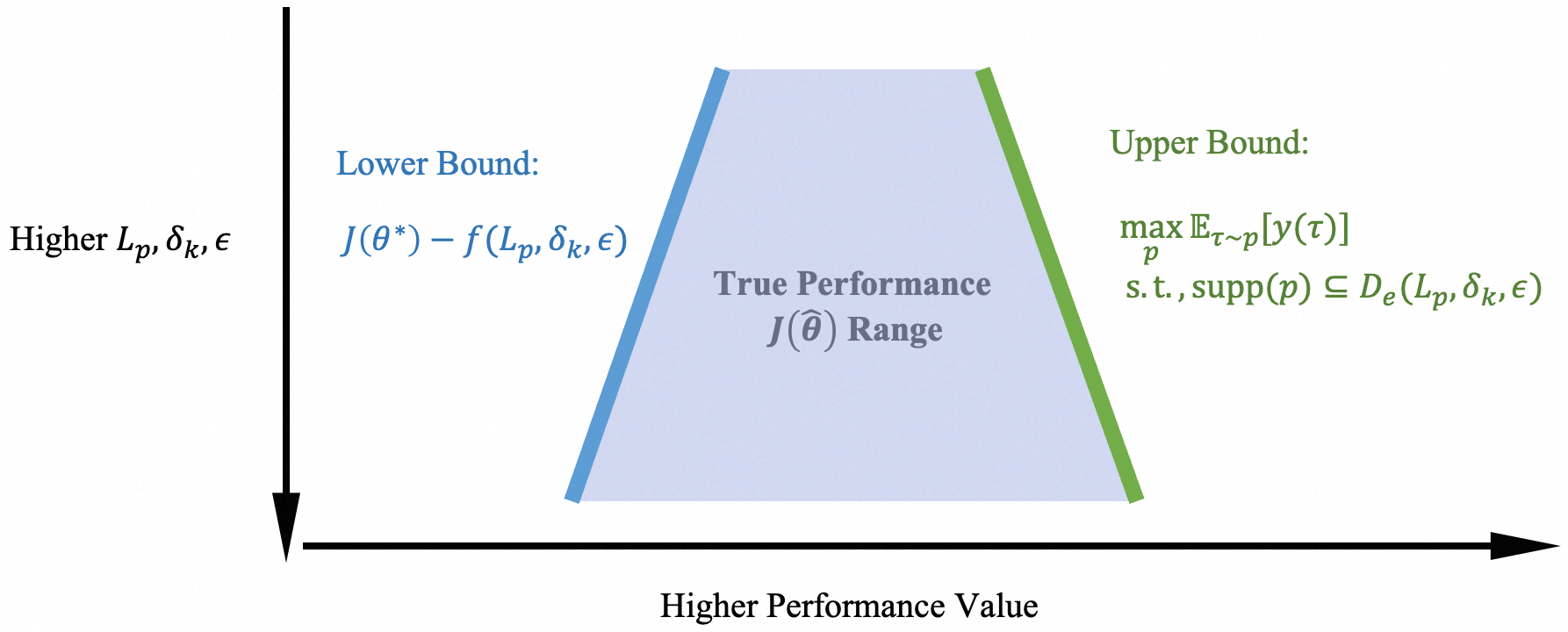}
    \caption{Performance range of the planner with respect to hyper-parameters $\delta_k, L_p$ and $\epsilon$. There is a trade-off in the selection of hyper-parameters: larger $\delta_k,L_p,\epsilon$ results in a smaller lower bound but a higher upper bound.}
    \label{fig:performance_range}
\end{figure}

\section{Theoretical Performance Range Tradeoff Discussion}
\label{app:theory_discussion}
We note that Theorem \ref{thm:sub_optimality} actually gives the lower bound of the planner's true performance $J(\hat{\theta})$, i.e.,
\begin{align}
    J(\theta^*)-\bigg(\underbrace{2\delta_D+(1+2k)\sqrt{T}R_m\bigg[\sqrt{\delta_M}+\sqrt{\delta_K}+(1+\epsilon)y_mL_p\bigg]\bigg)}_{\triangleq f(\delta_K, L_p,\epsilon)}\le J(\hat{\theta}),
\end{align}
where we denote that gap term $f$ as a function of hyper-parameters $\delta_K, L_p$ and $\epsilon$ that we can adjust. 
As $f(\delta_K, L_p,\epsilon)$ is monotonically increasing with respect to $\delta_K, L_p$ and $\epsilon$, smaller values of these terms result in a higher lower bound of $J(\hat{\theta})$. 
However, as illustrated in Fig. \ref{fig:kl_lipschitz}, these three terms also determine the planner's exploration range, denoted as $D_e(\delta_k,L_p, \epsilon)$. Specifically, higher values of $\delta_k,L_p, \epsilon$ indicate a larger exploration range, which thereby leads to a higher performance upper bound of the planner, i.e., 
\begin{align}
    J(\hat{\theta})\le \max_{p(\tau)} \mathbb{E}_{\tau\sim p(\tau)}[y(\tau)] \quad\mathrm{s.t.,}\;\; \text{supp}(p(\tau))\subseteq D_e(\delta_k, L_p, \epsilon).
\end{align}
This introduces a trade-off in the selection of hyper-parameters.
As illustrated in Fig. \ref{fig:performance_range}, higher $\delta_k,L_p,\epsilon$ results in a smaller lower bound but a higher upper bound. 
To this end, we conduct hyper-parameter tuning and give hyper-parameter determination in Appendix \ref{app:hyper_parameter_sensitivity}.


\section{AIGB-Pearl Algorithm Summary}
\label{app:aigb_pearl}

Algorithm \ref{algo} summarizes the training process of AIGB-Pearl.
Specifically, we compute the Lipschitz constants of $y(\tau)$ and $p_D(\tau|y)$ using the offline dataset $\mathcal{D}$ and, accordingly, determine the Lipschitz constraints for the evaluator and the planner, respectively.
Then, we sequentially perform evaluator learning, planner pretraining, and KL-Lipschitz-constrained score maximization for the planner.
The development of AIGB-Pearl is supported by ROLL \citep{wang2025reinforcement}. 

\subsection{Additional Designs For Evaluator Accuracy Enhancement}
\label{app:aigb_pearl_evaluator}
To further enhance the reliability of the trajectory evaluator, we design two additional techniques for the evaluator learning. Specifically, as described in the following, we (i) integrate an LLM embedding into its input feature for better representational capacity; and (ii) employ the pair-wise loss for better score estimation accuracy.
Fig. \ref{fig:aigb_pearl_evaluator} illustrates the complete evaluator learning losses.
The effectiveness of these two methods is studied in Appendix \ref{app:evaluator_accracy}.

\subsubsection{LLM Embedding Enhancement}
Motivated by the success of integrating user-specific features into recommendation systems \citep{chen2015recommender}, we incorporate advertiser-specific features into the trajectory evaluator to enhance its representational capacity and improve scoring accuracy.
Note that some advertiser-specific features are textual in nature—such as product titles, categories, and reviews—and are therefore difficult to incorporate directly into the vectorized trajectory $\tau$.
To address this, we construct a \emph{prompt} containing all such textual attributes and employ a pre-trained large language model (LLM), QWen2.5-1.5B-Instruct, with general world knowledge to extract $T$ representation, which we refer to as the \emph{LLM embedding}. 
Specifically, we use the output of the last hidden layer as the LLM embedding. 
This embedding is then used as an additional positional encoding in the evaluator.
The prompt template is given below: 

\begin{tcolorbox}[colframe=black,colback=gray!5,boxrule=0.5mm]
\footnotesize
\textbf{LLM Prompt Template.} 
I am an [\emph{advertising platform}] advertiser, operating the  [\emph{brand name}] brand in the category of  [\emph{category name}], classified as a  [\emph{advertiser tier}] tier advertiser.
I have a product titled [\emph{product name}] currently running in advertising campaigns. 
This product belongs to the leaf category of [\emph{leaf category}], priced at [\emph{product price}], and is positioned in the [\emph{price range}].
Its price ranks within the top  [\emph{price ranking in the leaf category}] \% in the leaf category.

\vspace{2mm}
Historical Average Performance:
The product generates an average of [\emph{average daily transactions}] daily transactions, with a GMV of [\emph{average daily GMV}], driven by advertising.
It receives an average [\emph{average daily impressions}] daily impressions from search and recommendation traffic,  [\emph{click numbers}] clicks,  [\emph{average daily BuyCnt}] BuyCnt, and a GMV of  [\emph{GMV}].
It ranks within the top  [\emph{sales volume ranking in leaf category}] \% in sales volume in the leaf category, with an average transaction value of  [\emph{average transaction value}].

\vspace{2mm}
Historical Time-based Average Performance:
This product has undergone continuous exposure to advertising for \emph{number of advertising days} days.
The average hourly advertising spend distribution per day (from 0:00 to 24:00) is  \emph{historical spend distribution sequence}.
The average GMV distribution across its category during this period is  \emph{historical GMV distribution sequence}.
The average daily advertising budget is  \emph{daily advertising budget}.
\end{tcolorbox}

\subsubsection{Pair-wise Loss}
Unlike human feedback scores used in LLM post-training \citep{rlhf}---which can suffer from subjective biases in their absolute valuations---the trajectory quality $y(\tau)$ has real physical meaning and is comparable among different trajectories. 
Therefore, we can adopt a hybrid point-wise and pair-wise loss for the score to capture the absolute value of  $y(\tau)$ and their relative preference, respectively.
This approach has demonstrated superior performance in recommender systems \citep{cao2007learning,lei2017alternating,wang2022mp2}.
Specifically, the pair-wise loss function can be expressed as:
\begin{align}
\label{equ:point_pair_loss}
   \mathbb{E}_{(\tau_w,\tau_l)\sim \mathcal{D}_p}\bigg[\log\sigma(\hat{y}_{\phi}(\tau_w)-\hat{y}_{\phi}(\tau_l))\bigg],
\end{align}
where the pair-wise loss is implemented by the typical Bradley-Terry (BT) function \citep{bradley1952rank}.
Here, $\mathcal{D}_p=\{(\tau_w, \tau_l, y(\tau_w), y(\tau_l))\}$ denotes a pair-wise dataset extracted from the offline dataset $\mathcal{D}$, where $\tau_w$ denotes the trajectory with higher trajectory quality, i.e., $y(\tau_w)\ge y(\tau_l)$.

\subsubsection{Overall Evaluator Loss}
Combining the previous techniques, the overall evaluator loss is:
\begin{align}
\label{equ:overall_evaluator_loss}
l_e(\phi)&= \underbrace{\mathbb{E}_{\tau\sim \mathcal{D}}\bigg[(\hat{y}_{\phi}(\tau)-y(\tau))^2\bigg]}_{\text{point-wise loss}}+\beta_4\underbrace{\mathbb{E}_{(\tau_w,\tau_l)\sim \mathcal{D}_p}\bigg[\log\sigma(\hat{y}_{\phi}(\tau_w)-\hat{y}_{\phi}(\tau_l))\bigg]}_{\text{pair-wise loss}}\notag\\
&+\beta_1\underbrace{\mathbb{E}_{\tau_1,\tau_2}\bigg[|\hat{y}_\phi(\tau_1)-\hat{y}_\phi(\tau_2)|-\sqrt{T}R_m\left\|\tau_1-\tau_2\right\|_F\bigg]_+}_{\text{Lipschitz penalty}},
\end{align}
where $\beta_4>0$ is a hyper-parameter.

\begin{figure}[t]
  \centering
\includegraphics[width=0.95\linewidth]{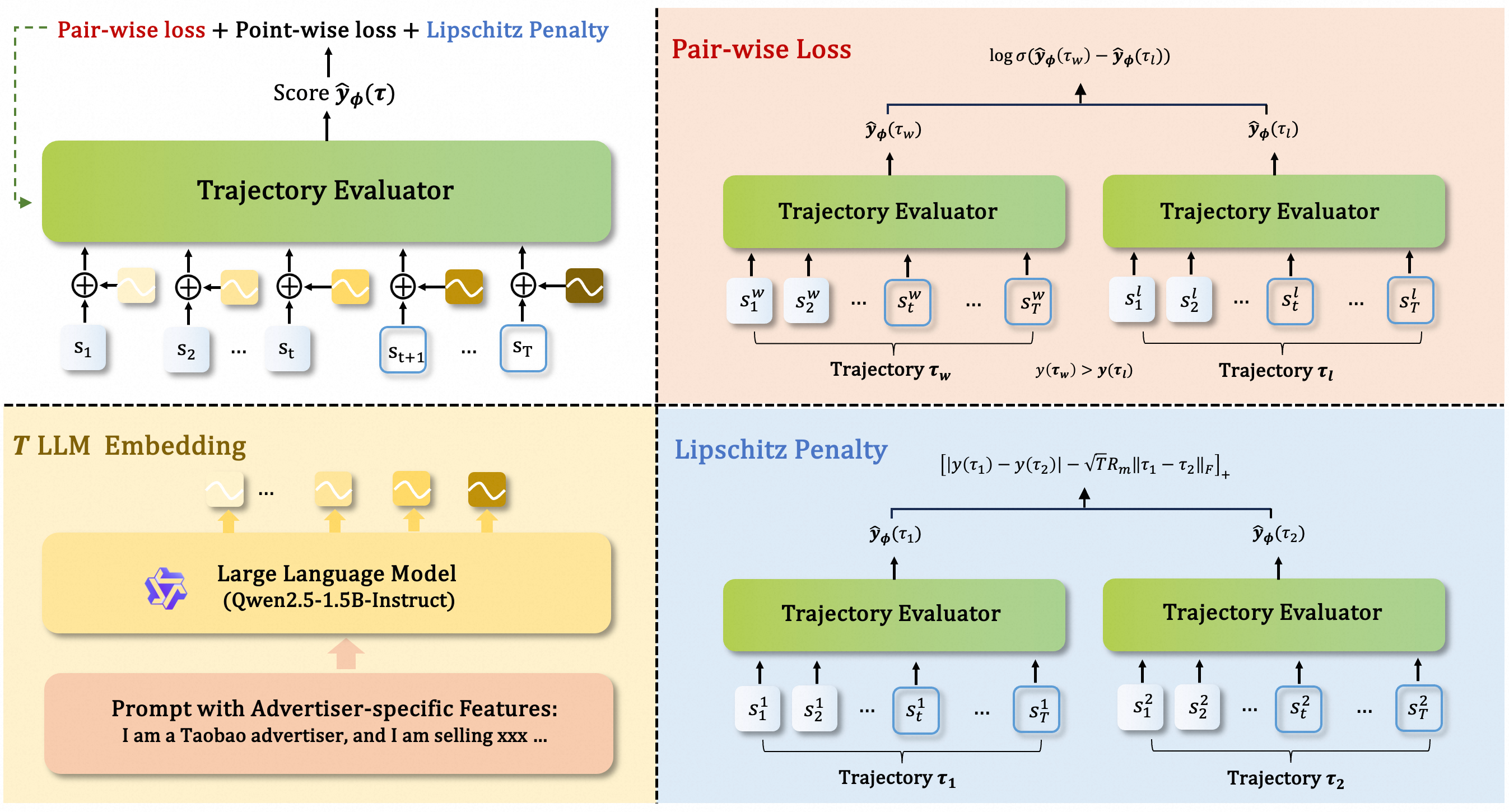}
  \caption{\textbf{AIGB-Pearl Evaluator Learning.} The evaluator loss is composed of three parts, including the point-wise loss, the pair-wise loss, and the Lipschitz penalty.
  The evaluator's input representation is augmented with an LLM embedding to incorporate semantic information from the advertiser's textual features.
  }
  \label{fig:aigb_pearl_evaluator}
\end{figure}

\begin{algorithm}[t]
\SetAlgoLined
\SetKwInOut{Input}{Input}
\SetKwInOut{Output}{Output}
\SetKwInOut{Initialization}{Initialization}
\Input{Offline dataset $\mathcal{D}$, desired condition $y^*$, hyper-parameters $\beta_1$, $\beta_2, \beta_3$.}
\Output{Optimized ${\theta}$ and $\phi$}
\Initialization{randomly initialized  planner parameter ${\theta}$, trajectory evaluator parameters {${\phi}$}}

\tcp{\textcolor{blue}
{\text{Determining the Lipschitz Value}}}

Calculate the Lipschitz value of $y(\tau)$ and $p_D(\tau|y)$ using the offline dataset $\mathcal{D}$.

Set the Lipschitz constraint value $L_e$ for the evaluator and $L_p$ for the planner to be bigger than the Lipschitz value of $y(\tau)$ and $p_D(\tau|y)$, respectively.

\tcp{\textcolor{blue}{\text{Training the trajectory evaluator}}}

\While{not converged}{

Update $\phi$ by minimizing Eq. \ref{equ:evaluator_loss};
}

\tcp{\textcolor{blue}{\text{Training the generative planner}}}
Warm start with pretrained planner $p_\theta$;

\While{not converged}{
Generate bidding trajectories $\tau\sim p_\theta(\tau|y^*)$;

Score generated trajectories with frozen ${\phi}$:
$\hat{y}_\phi(\tau)$;

Update $\theta$ by maximizing Eq. \ref{equ:planner_loss}.
}

\caption{AIGB-Pearl (Planning with EvaluAtor via RL)}
\label{algo}
\end{algorithm}

\section{Additional Experiments}
\label{app:exp}
\subsection{Simulated Experiment Settings}
\label{app:exp_simulate}

We include the detailed simulated experiment settings in Table \ref{app_tab:sim_environment}.
Specifically, we consider the bidding process in a day, where the bidding 
episode is divided into 96 time steps. Thus, the duration between
two adjacent time steps $t$ and $t+1$ is 15 minutes. The number
of impression opportunities between time steps $t$ and $t+1$ fluctuates
from 100 to 500.  
The {minimum and maximum} budgets of advertisers are 1000 Yuan and 4000 Yuan, respectively. 
The upper bound of the bid price is 10 Yuan, and the values of impressions {are positive}. 

\textbf{Hardware Resource.}
The simulated experiments are conducted based on an NVIDIA T4 Tensor Core GPU. We use 10 CPUs and 200G memory.

\begin{table}[t]
    \centering
     \caption{Settings of the simulated experiments.}
    \begin{tabular}{l|c}
    \toprule
       Parameters  & Values \\
    \midrule
       Number of advertisers  & 30 \\
       Time steps in an episode, $T$ & 96 \\
       Minimum number of impressions within a time step & 50 \\
       Maximum number of impressions  within a time step & 300 \\
       Minimum budget & 1000 Yuan \\
       Maximum budget & 4000 Yuan \\
       Value of impressions  & {$>0$} \\
       Minimum bid price, $\min\{av_i\}$ & 0 Yuan \\ 
       Maximum bid price, $\max\{av_i\}$ & 10 Yuan \\
       Maximum market price, $p_M$ & 10 Yuan \\
    \bottomrule
    \end{tabular}
    \label{app_tab:sim_environment}
\end{table}

\begin{table}[t]
    \centering
     \caption{Settings of the real-world experiments.}
    \begin{tabular}{l|c}
    \toprule
       Parameters  & Values \\
    \midrule
       Number of advertisers  & 6,000 \\
       Time steps in an episode, $T$ & 96 \\
       Minimum number of impressions  within a time step & 100 \\
       Maximum number of impressions within a time step & 2,500 \\
       Minimum budget & 50 Yuan \\
       Maximum budget & 10,000 Yuan \\
       Value of impressions  & {$>0$} \\
       Minimum bid price, $\min\{av_i\}$ & 0 Yuan \\ 
       Maximum bid price, $\max\{av_i\}$ & 25 Yuan \\
       Maximum market price, $p_M$ & 25 Yuan \\
    \bottomrule
    \end{tabular}
    \label{app_tab:real_environment}
\end{table}

\subsection{Real-world Experiment Settings}
\label{app:exp_real_world}
We include the detailed real-world experiment settings in Table \ref{app_tab:real_environment}.
Specifically, we consider the bidding process in a day, where the bidding 
episode is divided into 96 time steps. Thus, the duration between
two adjacent time steps $t$ and $t+1$ is 15 minutes. The number
of impression opportunities between time steps $t$ and $t+1$ fluctuates
from 100 to 2,500.  
The {minimum and maximum} budgets of advertisers are 50 Yuan and 10,000 Yuan, respectively. 
The upper bound of the bid price is 25 Yuan, and the values of impressions {are positive}. 

\textbf{Hardware Resource.}
The training process in the real-world experiments is conducted using 10 NVIDIA T4 Tensor Core GPUs in a distributed manner. For each distributional worker, we use 10 CPUs and 200 GB of memory.

\begin{figure}
    \begin{center}
  \includegraphics[width=0.5
\textwidth]{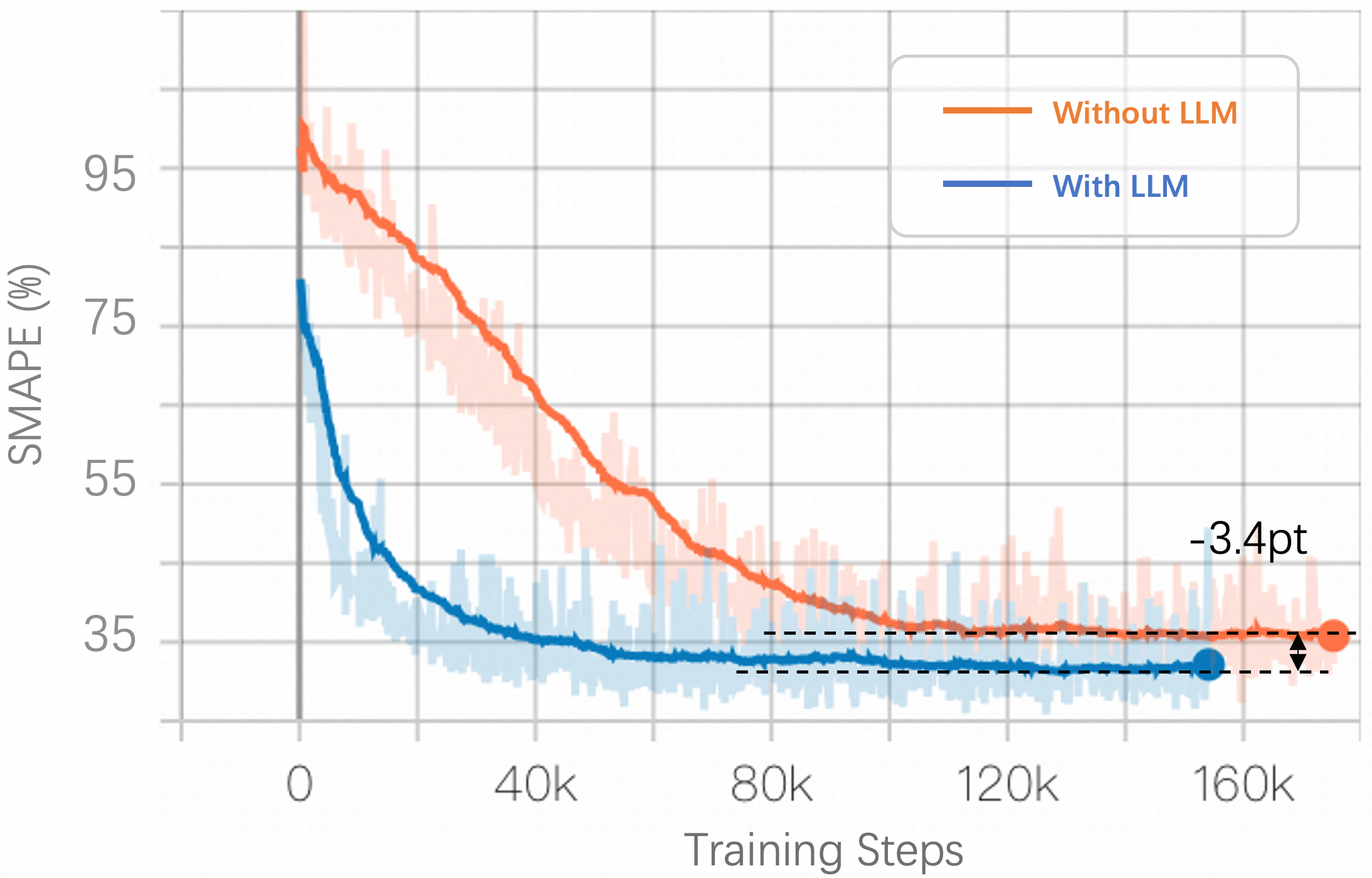}
  \end{center}
  \caption{The SMAPE training curves of the trajectory evaluator with and without LLM Embeddings.
  Incorporating LLM embeddings helps to achieve faster convergence and improved absolute accuracy.}
  \label{fig:ablation_llm}
\end{figure}

\subsection{Real-world Experiments on TargetROAS Bidding Problem}
\label{app:exp_targetroas}
In addition to the budget-constrained auto-bidding problem, we also apply the proposed AIGB-Pearl algorithm to a more challenging type of auto-bidding problem, named TargetROAS, with an extra ROI (Return on Investment) constraint. We evaluate our method in a real-world experiment on TaoBao involving 300k advertisers over 22 days. The offline dataset comprises 16 million trajectories of 800k advertisers.
The results are given in Table \ref{table:roas_res}. AIGB-Pearl achieves a \textbf{+5.1\%} improvement in GMV compared to the SOTA auto-bidding method, DiffBid, demonstrating its effectiveness in managing more complex and realistic constraints.
\setlength{\tabcolsep}{3pt}
\setlength{\abovecaptionskip}{-0.1pt}
\begin{table*}[t]
    \centering
    \caption{Overall performance of TargetROAS in real-world A/B test, involving 300k advertisers over 22 days.}
    \begin{tabular}{c|cccc}
    \toprule
       Methods & GMV & BuyCnt & ROI & Cost  \\
    
    \midrule
       \textbf{DiffBid} & 779,642,891 & 11,519,082 & 4.68 & 166,544,918 \\
       \textbf{AIGB-Pearl (ours)} & {819,550,812} & {11,886,501} & {4.70} & 174,234,673 \\
    \midrule
       $\bm \Delta$ & \textbf{+5.1\%} & \textbf{+3.2\%} & +0.5\% & \textbf{+4.6\%} \\
       \bottomrule
    \end{tabular}
     \label{table:roas_res}
\end{table*}

\subsection{Pathological Trajectory Behavior Explanation}
\label{app:pathological_behaviors}
In industrial practice, stable and effective metrics have been developed to evaluate pathological behaviors. 
For the case of the budget-constrained auto-bidding problem with bidding cycles structured as 24-hour episodes (T = 96 time steps), the following three key metrics are commonly used to identify pathological behaviors: 

\begin{itemize}
    \item \textbf{Excessive budget consumption}: there exists a time step $t$ such that the cost between time step $t$ and $t+1$ exceeds $10\%$ of the budget $B$;
    \item \textbf{Forward- (or Backward-) trending pacing}: the cost between time step $1$ and $24$ (or between time step $T-24$ and $T$) exceeds $40\%$ (or $40\%$) of the budget  $B$;
    \item \textbf{under-utilization of available budgets}: the total cost over the bidding episode is lower than $90\%$ of the budget $B$.
\end{itemize}

\subsection{Training Stability}
\label{app:exp_stability}

We present additional comparisons between the training curves of the offline RL with bootstrapping and those of AIGB-Pearl in Fig. \ref{fig:stability_score} and Fig. \ref{fig:stability_cost_ratio} concerning:
\begin{itemize}
\item \textbf{Scores}: the main performance index of the considered auto-bidding problem;
    \item \textbf{Cost Ratio}: the rate between the cost and the budget. A larger cost rate indicates a better performance.
\end{itemize}
It can be observed that the offline RL method tends to exhibit significant instability throughout training, with high variance across different seeds.
In contrast, AIGB-Pearl achieves much smoother and more consistent learning progress, demonstrating the improved training stability.

\begin{figure}[t]
    \centering
  \includegraphics[width=0.48
\textwidth]{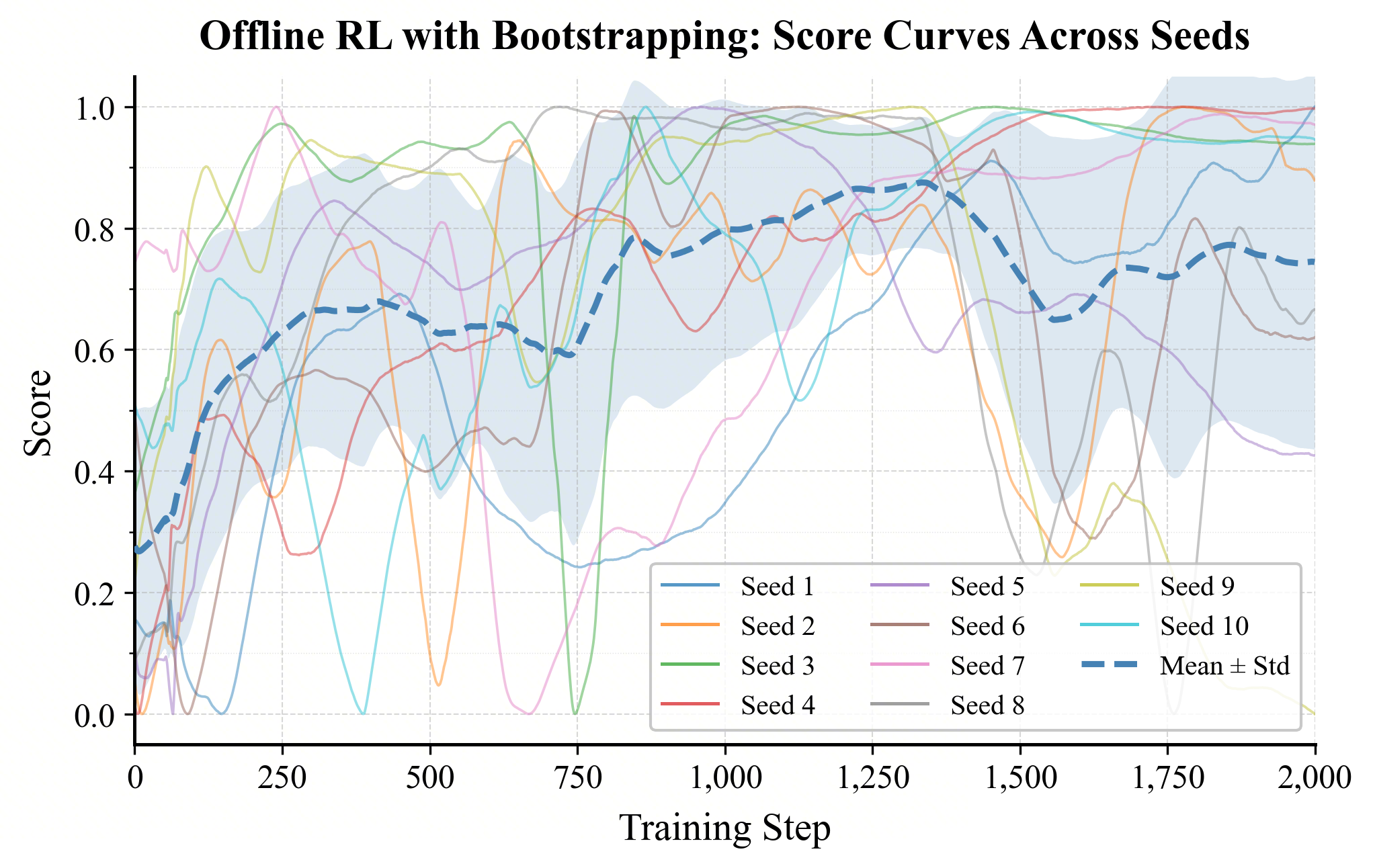}
  \hspace{0.05cm}
  \includegraphics[width=0.48
\textwidth]{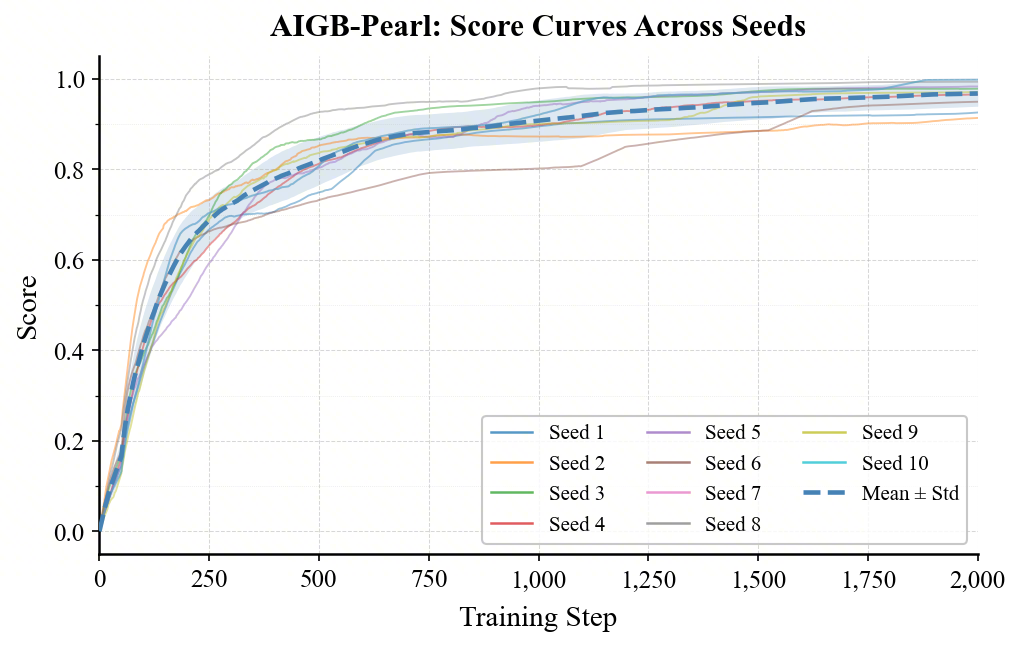}
  \caption{Learning curves of scores between offline RL with bootstrapping method and AIGB-Peral under 10 seeds.
  }
  \label{fig:stability_score}
\end{figure}

\begin{figure}[t]
    \centering
  \includegraphics[width=0.48
\textwidth]{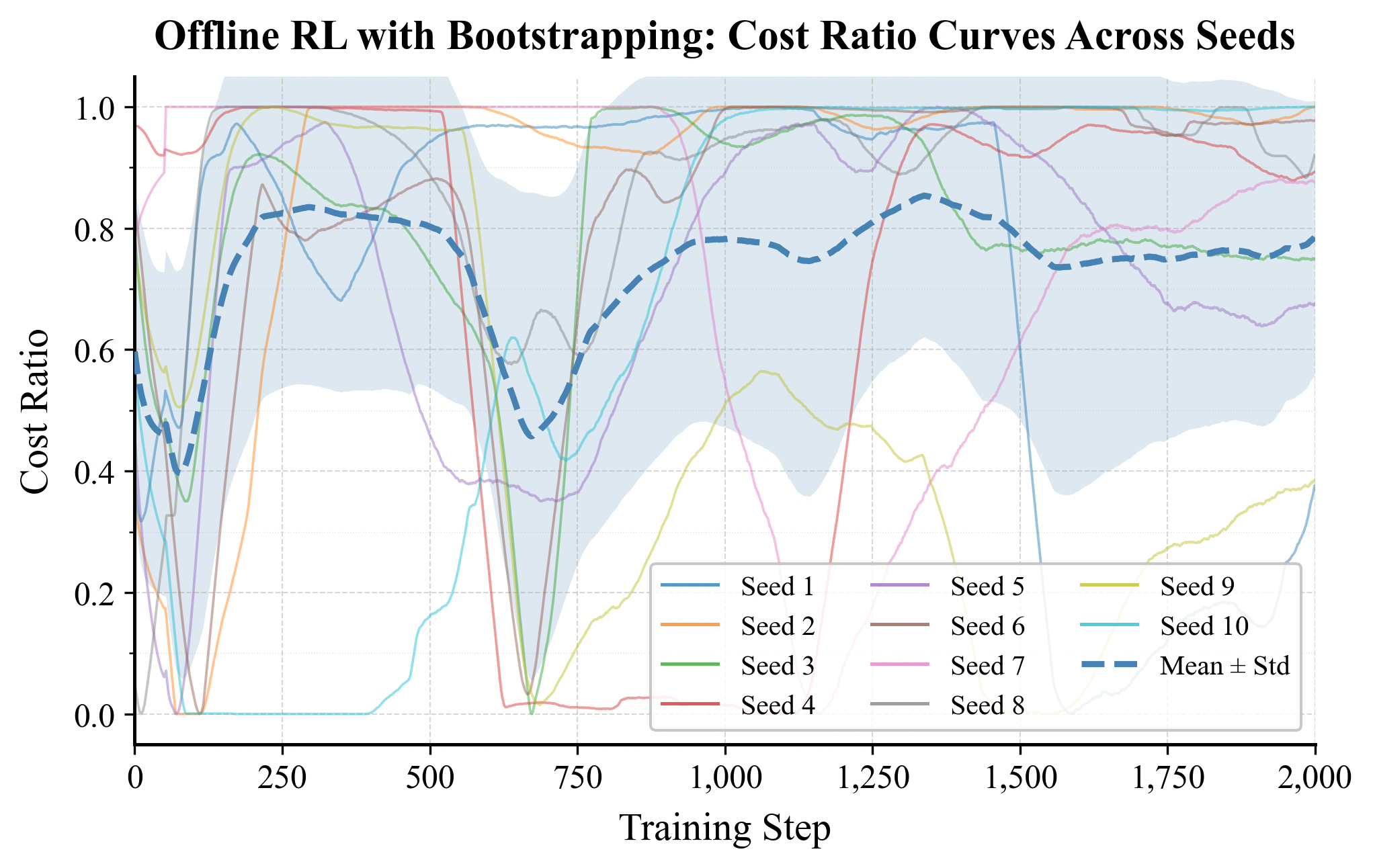}
  \hspace{0.05cm}
  \includegraphics[width=0.48
\textwidth]{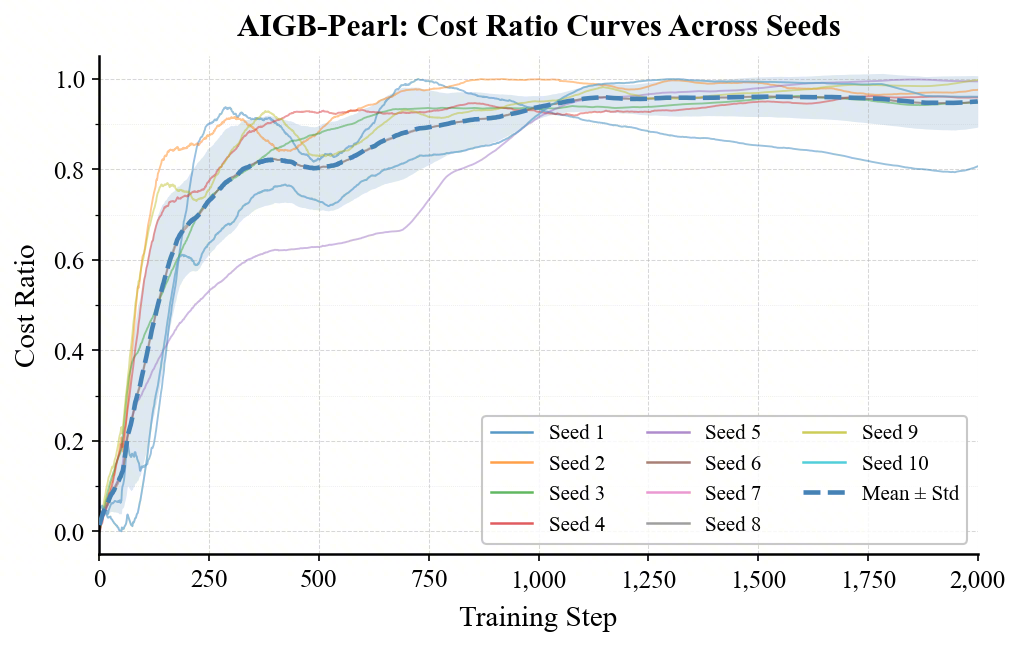}
  \caption{Learning curves of cost ratio between offline RL with bootstrapping method and AIGB-Peral under 10 seeds.
  }
  \label{fig:stability_cost_ratio}
\end{figure}





\subsection{Evaluator Accuracy}
\label{app:evaluator_accracy}
\textbf{Accuracy Metric.}
We evaluate the accuracy of the  evaluator along two dimensions, including the \emph{absolute accuracy}, reflecting how close the predicted scores are to ground truth qualities, and the \emph{order accuracy}, reflecting the correctness of relative rankings between trajectory pairs. Specifically, we use the \emph{symmetric mean absolute percentage error}, \textbf{SMAPE}, as the metric for the absolute accuracy and the \textbf{AUC}, defined as the ratio of correctly predicted ordinal pairs to the total number of pairs, as the metric for the order accuracy. 
The SMAPE ranges from $0\%$ to $200\%$, and the AUC ranges from $0\%$ to $100\%$.
Lower SMAPE and larger AUC indicate better evaluator accuracy. 

In the following, we investigate the effectiveness of using LLM embeddings and a pairwise loss for evaluator learning. 

\textbf{LLM Embedding Effectiveness.}
We examine the accuracy of the trajectory evaluator without using LLM embeddings, and the results are reported in the lines of ``w/o LLM'' in Table \ref{tab:ablation_SMAPE} and Table \ref{tab:ablation_AUC}.
It can be observed that LLM embeddings can improve both absolute and order accuracy.
Fig. \ref{fig:ablation_llm} compares the training progress with and without LLM embeddings in terms of the SMAPE metric.
As illustrated, the evaluator incorporating LLM embeddings converges faster and achieves a lower SMAPE than the one without LLM embeddings.
The performance gain stems from LLM embeddings' ability to encode high-level semantic information, thereby facilitating a more nuanced understanding of dependencies among sequential bidding states.

\textbf{Hybrid Point-wise and Pair-wise Losses Effectiveness.}
Table \ref{tab:ablation_SMAPE} and Table \ref{tab:ablation_AUC} present the SMAPE and AUC of the trajectory evaluator when trained with point-wise loss only, pair-wise loss only, and a combination of both. It can also be seen that using only pairwise loss yields significantly worse SMAPE performance, despite some improvement in AUC.
This suggests that while pairwise loss can enhance ranking consistency, it falls short of providing accurate absolute-value predictions.
When both point-wise and pair-wise losses are used together, the evaluator achieves lower SMAPE and higher AUC. 
This indicates that combining these two loss types not only improves absolute accuracy but also enhances order accuracy in trajectory evaluation.

\begin{table}
    \centering
    \setlength{\tabcolsep}{5.3pt}
    \begin{minipage}{0.47\textwidth}
       \caption{SMAPE results from ablation experiments on the trajectory evaluator.
    }
    \centering
    \footnotesize
    \begin{tabular}{c|ccc}
    \toprule
        SMAPE $\downarrow$ & Point-wise & Pair-wise  & Both \\
    \midrule
        w/o LLM & 43.55\% & 196.38\% & 35.00\% \\
        with LLM & 38.20\% & 196.00\% & \textbf{31.60}\%\\
    \bottomrule
    \end{tabular}
    \label{tab:ablation_SMAPE} 
    \end{minipage}
    \hspace{0.15in}
    \setlength{\tabcolsep}{5.3pt}
    \begin{minipage}{0.47\textwidth}
        \caption{AUC results from ablation experiments on the trajectory evaluator.
    }
    \centering
    \footnotesize
    \begin{tabular}{c|ccc}
    \toprule
        AUC $\uparrow$ & Point-wise & Pair-wise  & Both \\
    \midrule
        w/o LLM & 61.57\% & 64.91\% & 70.91\% \\
        with LLM & 65.00\% & 71.30\% & \textbf{75.10}\%\\
    \bottomrule
    \end{tabular}
    \label{tab:ablation_AUC}
    \end{minipage}
    \vspace{-3mm}
\end{table}

\subsection{Empirical performance with general offline data distributions}
Note that in many real-world auto-bidding systems, including the one considered in the paper, due to operational safety constraints, the online-deployed bidding policy is typically a single fixed model, and the offline dataset is collected exclusively from this single policy over multiple days, where an advertiser contributes a single trajectory per day. For example, in the considered auto-bidding system, the online-deployed baseline policy is a conditional generative model that, given a given advertiser and identical conditions, generates identical trajectory plans each day. The variation across different trajectories of the same advertiser in the offline dataset is solely due to stochastic environmental factors (e.g., traffic fluctuations). Since these exogenous perturbations are typically the sum of many independent impression-level sources of noise, the resulting trajectory deviations can be reasonably approximated as a Gaussian distribution.

To demonstrate the broad applicability of the proposed algorithm,  we evaluate its performance in settings where multiple policies are used for data collection. Specifically, we collect trajectories in the simulated environment using nine distinct bidding policies, thereby constructing an offline dataset with a multi-modal distribution that violates the Gaussian distribution. The empirical results are presented in Table \ref{table:LSI_not_hold}.

\setlength{\tabcolsep}{5pt}
\begin{table*}[h]
    \centering
    \small
    \caption{Empirical performance with multiple data-collection policies.}
    \begin{tabular}{c|ccc}
    \toprule
       Methods & GMV & ROI & Cost  \\
    
    \midrule
       \textbf{DiffBid} & 548.5 & 5.00 & 109.4 \\
       \textbf{AIGB-Pearl (ours)} & {575.7} & {5.04} & {114.6} \\
    \midrule
       $\bm \Delta$ & \textbf{+4.9\%} & \textbf{+0.7\%} & \textbf{+4.2\%}  \\
       \bottomrule
    \end{tabular}
     \label{table:LSI_not_hold}
\end{table*}

It can be observed that AIGB-Pearl still outperforms AIGB by 4.9\%, indicating that its performance is robust to the offline dataset's specific distribution.

\subsection{Hyper-parameter Tuning}
\label{app:hyper_parameter_sensitivity}
We conduct sensitivity experiments with respect to the hyper-parameter $\delta_k$.
For hyper-parameter $L_p$, we leverage the lower bound given in Eq. \ref{equ:lower_bound_Lp} in the main experiments. For the hyper-parameter $\epsilon$, we use the same empirical value of $5\%$ as in AIGB in the main experiments, which is typically set based on operational requirements. 

Specifically, in the planner loss given in Eq. \ref{equ:planner_loss}, $\beta_2$ is the penalty factor corresponding to the KL constraint. Actually, we control the KL divergence $\delta_k$ by tuning $\beta_2$. Table. \ref{table:hyper_parameter_KL} gives the hyper-parameter tuning results. It can be observed that, as long as $\delta_k$ is neither too large (in which case the algorithm degenerates into AIGB with a pure demonstration likelihood maximization term) nor too small (e.g., $\beta_2=0$, which completely removes the demonstration likelihood maximization), the proposed method consistently outperforms the original AIGB.
This also validates the discussion on performance bounds in Appendix \ref{app:theory_discussion}, which demonstrates that a moderate $\delta_k$ balances the lower and upper bounds, thereby yielding optimal performance.

\setlength{\tabcolsep}{5pt}
\begin{table*}[h]
    \centering
    \small
    \caption{Hyper-parameter tuning with respect to KL constraint $\delta_k$.}
    \begin{tabular}{c|cccccccc}
    \toprule
       $\beta_2$ Setting & 0 (no KL)&0.1&0.5&1.0&2&5&10& $\infty$ (AIGB, only KL) \\
    \midrule
       $\delta_k$ values & 0.85 & 0.51 & 0.50 &0.49&0.43&0.39&0.36&0.32\\
       \midrule
       GMV & {548.5} & {563.8} & {571.3}&\textbf{574.2}&573.9&570.7&567.6&556.3 \\
       \bottomrule
    \end{tabular}
     \label{table:hyper_parameter_KL}
\end{table*}

\section{Extending AIGB-Pearl to first-price auctions}
\label{app:extension_to_first_price}
We note that the proposed method remains effective in first-price auctions with a proper adaptation. Specifically, unlike second-price auctions where the optimal bid for impression $i$ takes the form $\text{bid}_i=av_i$, in first-price auctions, the optimal bid for impression $i$ is given by $\text{bid}_i=\min (av_i, p_i)$, which typically involves an extra bid shading method to predict the winning price \citep{gligorijevic2020bid,wu2015predicting}. Equipped with an off-the-shelf bid shading method (whose design is beyond the scope of this work), the auto-bidding problem in a first-price auction remains an offline sequential decision problem, i.e., making decisions over an $a$-sequence, to which the proposed method applies directly.

To validate the effectiveness of the proposed method under first-price auctions, we additionally conduct a simulated first-price auction experiment against the state-of-the-art AIGB method, where all methods are equipped with the same bid shading method. The results are presented in Table \ref{table:first_price}, demonstrating the effectiveness of our proposed method.

\begin{table*}[t]
    \centering
    \small
    \caption{Empirical performance in the simulated experiments with a first-price auction.}
    \begin{tabular}{c|ccc}
    \toprule
       Methods & GMV &  ROI & Cost  \\
    
    \midrule
       \textbf{DiffBid} & 1,546  & 5.13& 301 \\
       \textbf{AIGB-Pearl (ours)} & {1,611}  & {5.18} & 311\\
    \midrule
       $\bm \Delta$ & \textbf{+4.2\%} & \textbf{+1.0\%} & \textbf{+3.3\%} \\
       \bottomrule
    \end{tabular}
     \label{table:first_price}
\end{table*}

\section{Extending AIGB-Pearl to Online Setting Discussion}
We note that AIGB-Pearl can be extended to online settings when equipped with a safe online exploration policy. Specifically, due to safety constraints, the auto-bidding policy during training cannot interact directly with the live advertising system; only safe exploration policies are permitted to collect data online \citep{sorl}. Consequently, an online auto-bidding framework typically involves two parts:
\begin{itemize}
    \item a safe online exploration policy, which is a well-established component in existing work \citep{sorl} and beyond the scope of this paper;
    \item an offline policy training method that leverages the data collected online.
\end{itemize}
AIGB-Pearl can be directly applied as the offline policy training method within the online framework without modification. In practice, due to the safety and stability concerns, many industrial auto-bidding systems adopt an offline optimization paradigm. For this practical reason, we focus on the offline setting in this work.

\section{LLM Usage}

The authors have used Large Language Models (LLMs) exclusively for grammar checking and lexical refinement during the writing process. No LLM-generated content, data analysis, or substantive contributions to the research methodology, results, or conclusions are involved in this work.

\subsubsection*{Acknowledgments}
This work was supported by Alibaba Group through the Alibaba Innovative Research Program.

\end{document}